\documentclass{article}

\PassOptionsToPackage{numbers}{natbib}
\usepackage[preprint]{neurips_2024}

\usepackage{amsmath,amsfonts,bm}

\def\eqref#1{equation~\ref{#1}}

\def\1{\bm{1}}

\DeclareMathAlphabet{\mathsfit}{\encodingdefault}{\sfdefault}{m}{sl}
\SetMathAlphabet{\mathsfit}{bold}{\encodingdefault}{\sfdefault}{bx}{n}

\def\gA{{\mathcal{A}}}

\def\gS{{\mathcal{S}}}

\newcommand{\E}[2]{\mathbb{E}_{#1} \left[#2\right]}

\DeclareMathOperator*{\argmax}{arg\,max}

\usepackage[utf8]{inputenc} 
\usepackage[T1]{fontenc}    
\usepackage{hyperref}       
\usepackage{url}            
\usepackage{booktabs}       
\usepackage{amsfonts}       
\usepackage{nicefrac}       
\usepackage{microtype}      
\usepackage{xcolor}         
\usepackage{graphicx}
\usepackage{amsmath}
\usepackage{amsthm}
\usepackage{algorithmic}
\usepackage{algorithm}
\usepackage{diagbox}
\usepackage{dsfont}
\usepackage{bbm}

\theoremstyle{definition}
\newtheorem{dfn}{Definition}[section]
\newtheorem{prop}[dfn]{Proposition}
\newtheorem{lem}[dfn]{Lemma}
\newtheorem{thm}[dfn]{Theorem}

\newtheorem{ass}[dfn]{Assumption}

\title{Strategically Conservative Q-Learning}

\author{
  Yutaka Shimizu \\
  University of California Berkeley\\
  \texttt{purewater0901@berkeley.edu} \\
  \And
  Joey Hong \\
  University of California Berkeley\\
  \texttt{joey\_hong@berkeley.edu} \\
  \And
  Sergey Levine \\
  University of California Berkeley\\
  \texttt{sergey.levine@berkeley.edu} \\
  \And
  Masayoshi Tomizuka \\
  University of California Berkeley\\
  \texttt{tomizuka@berkeley.edu} \\
}

\begin{document}

\maketitle

\begin{abstract} \label{sec:abstract}
  Offline reinforcement learning (RL) is a compelling paradigm to extend RL's practical utility by leveraging pre-collected, static datasets, thereby avoiding the limitations associated with collecting online interactions. 
  The major difficulty in offline RL is mitigating the impact of approximation errors when encountering out-of-distribution (OOD) actions; doing so ineffectively will lead to policies that prefer OOD actions, which can lead to unexpected and potentially catastrophic results. Despite the variety of works proposed to address this issue, they tend to excessively suppress the value function in and around OOD regions, resulting in overly pessimistic value estimates.
  In this paper, we propose a novel framework called Strategically Conservative Q-Learning (SCQ) that distinguishes between OOD data that is easy and hard to estimate, ultimately resulting in less conservative value estimates. 
  Our approach exploits the inherent strengths of neural networks to interpolate, while carefully navigating their limitations in extrapolation, to obtain pessimistic yet still property calibrated value estimates. 
  Theoretical analysis also shows that the value function learned by SCQ is still conservative, but potentially much less  so than that of Conservative Q-learning (CQL). 
  Finally, extensive evaluation on the D4RL benchmark tasks shows our proposed method outperforms state-of-the-art methods. 
  Our code is available through \url{https://github.com/purewater0901/SCQ}.
\end{abstract}

\section{Introduction} \label{sec:introduction}
Reinforcement learning (RL)\cite{Sutton1998RL} is a powerful paradigm for decision-making and control, but requires extensive exploration of the environment that can be expensive and dangerous. Offline reinforcement learning \cite{levine2020offline} eliminates such need, allowing learning from static datasets from any unknown behavior policy, without any online environment interaction.
However, designing effective offline RL algorithms comes with new challenges, particularly in counterfactual reasoning outside of the data distribution. Without the possibility of correction via online interaction, value estimation for out-of-distribution (OOD) data is often overly optimistic \cite{fujimoto2018addressing, bear}, leading to extrapolation errors that harm policy learning.  

Existing approaches tackle OOD overestimation in two major ways, by either constraining the learned policy or estimated values. Policy constraint directly constrains the learned policy to be close to the behavior policy \citep{bear}. 
In contrast, the value constraint technique aims to regularize the value function, particularly assigning low values to OOD actions ~\citep{cql,iql}. 
While both paradigms avoid exploitation of OOD actions by the learned policy, they often sacrifice generalization performance by being overly conservative, even in regions of the data where actions can be evaluated accurately ~\citep{calql}.

Hence, what many existing offline RL algorithms lack is a proper balance of conservatism and generalization. Specifically, Conservative Q-learning ~\citep{cql} (CQL) will penalize the value of all actions taken by the learned policy, including ones that are in-distribution; another popular algorithm, Implicit Q-learning ~\citep{iql} (IQL), constrains the value function to be estimated within the data support, but in doing so completely neglects any possibility for generalization outside of the dataset.
The consequences of this over-penalization are illustrated in the first and second diagrams in Fig.\ref{fig:sql-image}. 
While it is important to rely on conservatism to avoid catastrophic OOD behavior, we believe that it is important to discriminate between OOD regions. Particularly, we propose \emph{strategic conservatism}, that will \emph{only penalize the value function in OOD regions where generalization is improbable}.

In this paper, we propose Strategically Conservative Q-learning (SCQ) to address this problem. SCQ is a member of the value constraint methods and inherits some key ideas from CQL \cite{cql}.
A distinctive feature of SCQ compared to other value constraint methods lies in its approach to minimizing the value function not across all OOD regions, but specifically ones sufficiently far from the data distribution. This idea is inspired by an important characteristic of neural networks (NN), that is, NN is not good at extrapolation but is adequate at interpolation\cite{Haley1992Extrapolation, Barnard1992Extrapolation, Arora2019FineGrained, xu2021how, Pete2022ImplicitBehaviorCloning}. 
In other words, the estimated values derived from functions approximated by neural networks have infinitesimal errors from the true values inside and around the data distribution. By leveraging this idea we can show that the estimated state-action values derived from SCQ are the $\varepsilon$-point-wise lower bound of the true values. In addition to its point-wise lower bound property, we demonstrate that the state value obtained by SCQ is less conservative than that derived from CQL in a linear MDP environment, while still serving as a lower bound of the true state value function. These properties theoretically support our statement that SCQ is less conservative than CQL. The right-most picture of Fig.\ref{fig:sql-image} shows these properties of the proposed method. From the picture, we can see our algorithm only penalizes state-action values at points that are far from the dataset.

SCQ can be adapted to many model-free offline RL algorithm. In this paper, we build SCQ on Soft-Actor-Critic (SAC) \cite{sac, sac2, sac3} and evaluate it on Mujoco locomotion and Antmaze tasks against existing state-of-the-art offline RL methods. Additionally, we reduce the dataset size to test its robustness toward data size. Finally, ablation studies of our method also demonstrate that the positive effects of our proposed objective cannot be replicated by architectural changes such as layer normalization, which has recently become a widely recognized as an effective way to deal with distribution shift ~\citep{rebrac}.

\begin{figure}[t]
  \centering
  \includegraphics[scale=0.35]{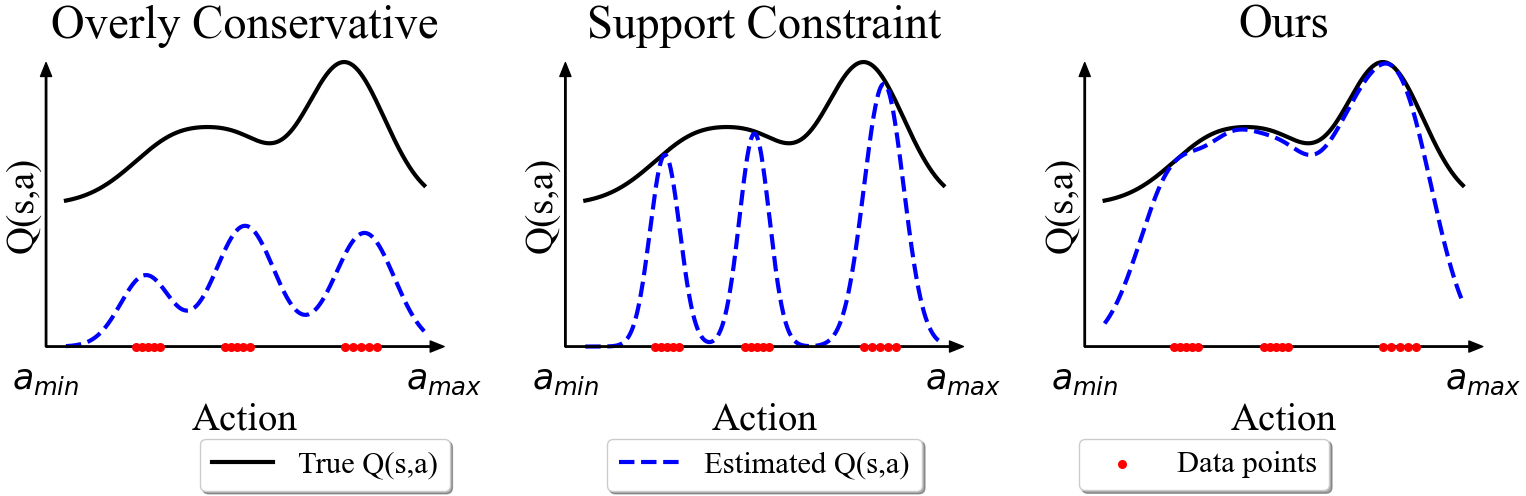}
  \caption{Comparison of true Q-values and estimated Q-values across different approaches. One example overly conservative method is CQL\cite{cql}, and IQL\cite{iql} is one instance of support constraint. Unlike these methods, our method leverages NN interpolation ability and guarantees point-wise lower bounds of true Q-values.}
  \label{fig:sql-image}
\end{figure}

\section{Related Work}
There are mainly two methods for finding the optimal policy from a static dataset: imitation learning and offline reinforcement learning. Although imitation learning\cite{daggar, dart, Atkeson1997RobotLearning,  Mozer1996Learning, Pastor2009Learning} does not have to design complex reward functions and handle OOD extrapolation error problems, it is hard to obtain a policy that surpasses the behavior policy.

Offline reinforcement learning, on the other hand, can get an optimal policy from a noisy complex dataset, but it needs to address out-of-distribution problems. One way is to constrain the learned police to the behavior policy that is used to collect the dataset\cite{Le2019BatchPL, Kumar2019Stabilizing, doge}. While some methods\cite{Fujimoto2018OffPolicyDR, matsushima2021deploymentefficient} meticulously model and constrain the learned policy, others use explicit divergence constraints\cite{Kumar2019Stabilizing, wu2020behavior, td3-bc, Zhang2021BRAC, Dadashi2021Offline} or implicit divergence constraints\cite{peng2021advantageweighted, awac, xu2022apolicyguided} to keep the learned policy close to the behavior policy. Value constraint methods, in contrast, penalize state-action value functions in out-of-distribution regions to extract optimal policies\cite{cql, Ilya2021Fisher, uwac, combo, mcq, lee2021offlinetoonline, rnd-sac, calql}. 

MCQ\cite{mcq} explicitly separates OOD actions from IDD actions and minimizes state values in OOD regions. DOGE\cite{doge} learns a distance function to determine the OOD actions and penalizes the OOD actions when it updates the learned policy. ReBrac\cite{rebrac} extends the TD3-BC\cite{td3-bc} by adding layer normalization and tuning the parameters. While these methods minimize state-value functions outside of the support of the in-data distribution, their estimated values still likely become overly conservative. A different approach is to learn Q functions without querying data from OOD regions\cite{iql, wang2018Advances, chen2020bail}. Although these methods successfully avoid dealing with OOD data, the resulting policy tends to converge on a sub-optimal level due to excessively pessimistic state-action value estimations.

\section{Preliminaries} \label{sec:preliminaries}
We first give a brief introduction to reinforcement learning and associated notations. After that, we discuss the limitations of existing offline reinforcement learning methods. 

\subsection{Reinforcement Learning} \label{subsec:rl-preliminaries}

Reinforcement Learning (RL) aims to optimize agents that interact with a Markov Decision Process (MDP) defined by a tuple $(\gS, \gA, P, r, \mu_1, \gamma)$, where $\gS$ represents the set of all possible states, $\gA$ is the set of possible actions,  $\mu_1$ is the initial state distribution, and $\gamma$ is the discount factor. When action $a \in \gA$ is executed at state $s \in \gS$, the next state is generated according to $s' \sim P(\cdot | s, a)$, and the agent receives stochastic reward with mean $r(s, a) \in \mathbb{R}$.

The Q-function $Q^\pi(s, a)$ for a policy $\pi(\cdot | s)$ represents the discounted long-term reward attained by executing $a$ given observation history $s$ and then following policy $\pi$ thereafter. 
$Q^\pi$ satisfies the Bellman recurrence: 
$Q^\pi(s, a) = \mathbb{B}^\pi Q^\pi(s, a) = r(s, a) + \gamma \E{s' \sim P(\cdot|s, a), a' \sim \pi(\cdot|s')}{Q_{h+1}(s', a')}\,.
$
The value function $V^\pi$ considers the expectation of the Q-function over the policy $V^\pi(h) = \E{a \sim \pi(\cdot | s)}{Q^\pi(s, a)}$. Meanwhile, the Q-function of the optimal policy $Q^*$ satisfies: $Q^*(s, a) = r(s, a) + \gamma \E{s' \sim P(\cdot|s, a)}{\max_{a'} Q^*(s', a')}$, and the optimal value function is $V^*(s) = \max_{a} Q^*(s, a)$. Finally, the expected cumulative reward is given by $J(\pi) = \E{s_1 \sim \mu_1}{V^\pi(s_1)}$.
The goal of RL is to optimize a policy $\pi(\cdot \mid s)$ that maximizes the cumulative reward 
$
J(\pi) = \E{\mu_1}{V^\pi(s_1)}
$.

In offline RL, we are provided with a dataset $\mathcal{D} = \{(s_i, a_i, r_i, s'_{i})\}_{i=1}^N$ of size $|\mathcal{D}| = N$.  We assume that the dataset $\mathcal{D}$ is generated i.i.d. from an effective behavior policy $\pi_\beta(a|s)$. 
The absence of online interactions in offline RL impedes the acquisition of new data, leading to distributional shift.

\subsection{Actor-Critic Methods}
Actor-critic methods are a common way to perform value-based RL. Algorithms in this family consists of two steps to learn optimal policies: policy evaluation and policy improvement step. 
The policy evaluation step computes $Q^\pi(s,a)$ for policy $\pi$ by fitting it to its Bellman backup target value:
\begin{equation}
Q^\pi(s, a) = B^\pi Q^\pi(s,a) := r(s,a) + \gamma \mathbb{E}_{s' \sim P(\cdot|s,a), a' \sim \pi(\cdot|s')}\left[Q^\pi(s',a')\right].
\end{equation}
Subsequently, the policy is updated in the policy improvement step:
\begin{equation} \label{eq:actor-update}
    \pi = \underset{\pi}{\operatorname{argmax}} \ \mathbb{E}_{s\sim \mathcal{D}, a \sim \pi(\cdot |s)} \left[ Q^\pi(s,a) \right]
\end{equation}
Given that true dynamics and reward of the MDP are unknown, we use the empirical Bellman operator $\hat{B}^{\pi}$, which utilizes a single sample from dataset $\mathcal{D}$ to compute the target value. In the presence of function approximation, we model $Q^\pi, \pi$ as $Q_{\theta}, \pi_{\phi}$, with $\theta$ and $\phi$ as their parameters. The policy evaluation and improvement steps at iteration $k$ can be expressed as:
\begin{equation} \label{eq:ct-bellman-update}
\begin{split}
    \theta_{k+1} \leftarrow &\underset{\theta}{\operatorname{argmin}}\; \mathbb{E}_{(s, a, r, s') \sim \mathcal{D}} \left[ \left(Q_{\theta}(s,a) - r(s,a) - \gamma \mathbb{E}_{a'\sim \pi_{\phi_k}(a'|s)}[Q_{\bar{\theta}_k}(s,a)] \right)^2 \right] \\
    \phi_{k+1} \leftarrow  &\underset{\phi}{\operatorname{argmax}} \; \mathbb{E}_{s\sim \mathcal{D}, a\sim\pi_{\phi}(\cdot |s)} \left[ Q_{\theta_{k+1}}(s,a) \right]\,,
\end{split}
\end{equation}
where $\bar{\theta}_k$ are target parameters that are a slow-moving copy of $\theta_k$.

\subsection{Conservative Q-learning}
Conservative Q-learning (CQL)\cite{cql} is an actor-critic algorithm that aims to constrain values in out-of-distribution areas to avoid overestimation. CQL updates modifies the policy evaluation step in actor-critic algorithms to instead solving the following optimization problem:
\begin{equation} \label{eq:cql-q-update}
\begin{split}
    \theta^{k+1} &\leftarrow \underset{\theta}{\operatorname{argmin}} \; \alpha  \left( \mathbb{E}_{s, a \sim D,a \sim \pi_{\phi_k}(a|s)} \left[ Q_{\theta}(s,a) \right] - \mathbb{E}_{s, a \sim D} \left[ Q_{\theta}(s,a) \right] \right) \\
&\qquad\qquad + \frac{1}{2} \mathbb{E}_{(s, a, r, s') \sim \mathcal{D}} \left[ \left(Q_{\theta}(s,a) - r(s,a) - \gamma \mathbb{E}_{a'\sim \pi_{\phi_k}(a'|s)}[Q_{\bar{\theta}_k}(s',a')] \right)^2 \right]
\end{split}
\end{equation}
where $\alpha$ is a hyperparameter to change the conservativeness. Under CQL, the estimated Q-values are guaranteed to be pessimistic, but often too much so, leading to a overly conservative policy. This conservativeness comes from the over-penalization of Q-values inside and outside of the dataset.

\section{Strategically Conservative Q Learning} \label{sec:scq-main}
In this section, we begin by detailing how Q-values are learned in our algorithm, dubbed Strategically Conservative Q-Learning (SCQ). We then theoretically demonstrate several advantageous properties of SCQ. 

\subsection{Q-learning with Strategic Conservatism}
Our primary goal is to minimize Q-values over OOD actions, while also leveraging the interpolation capability of neural networks as Q-functions. 
Intuitively, we want to decrease the Q-values of OOD actions that are clearly too far from the dataset to be estimated property, but do not want to penalize Q-values for actions when they can likely be interpolated by the neural network from actions within the dataset. We define the OOD actions in the following way:

\begin{dfn} [Out-of-distribution actions] \label{def:ood-actions}
    The out-of-distribution (OOD) actions for a state $s$ are defined as
    \begin{equation} \label{eq:ood-action-definition}
        A_{\text{ood}}(s) = \left\{ a \mid \left|\left| a - \text{Proj}_{\mathcal{D}}(a \mid s) \right|\right|_2 \geq \delta \right\}
    \end{equation}
    where $\delta$ is a threshold to choose the distance threshold for OOD actions, and we define the projected action to the dataset as $\text{Proj}_{\mathcal{D}}(a \mid s) := \arg \min_{(s, a') \in \mathcal{D}} \| a - a' \|$ the action that was taken at state $s$ that is closest to $a$ in the dataset.
\end{dfn}

Fix policy $\pi$, and let $\hat{Q}$ be an approximation of the Q-function under the policy $Q^\pi$. 
We define a new policy $\pi_{\text{ood}}$ that depends on policy $\pi$ in the following manner: $\pi_{\text{ood}}(a | s) \propto \pi(a | s) \mathbbm{1}[a \in A_{\text{ood}}(s)]$. Hence, $\pi_{\text{ood}}$ can be interpreted as the policy $\pi$ but with support constrained entirely to OOD actions.
With the policy $\pi_{\text{ood}}$ defined as above, we consider the following objective to solve for the estimated Q-values:
\begin{equation} \label{eq:scq-q-update}
\hat{Q} = \underset{Q}{\operatorname{argmin}} \; \alpha \, \mathbb{E}_{s \sim \mathcal{D}, a \sim \pi_{\text{ood}}(\cdot | s)} \left[ Q(s, a) \right] + \frac{1}{2} \mathbb{E}_{(s, a, r, s') \sim \mathcal{D}} \left[ \left( Q(s, a) - \hat{B}^{\pi} {\hat{Q}}(s, a) \right)^2 \right]
\end{equation}
where $\alpha$ is a hyperparameter used to adjust the conservativeness of the Q-functions. The first term minimizes the Q-values for out-of-distribution actions, while the second term represents the standard Bellman update. Intuitively, SCQ minimizes Q-values specifically for actions that are far from the dataset and uses interpolated values for points within and near the dataset. This property distinguishes SCQ from previous works, which minimize Q-values even in regions that are inside and close to the dataset. This can be particularly detrimental when the Q-function is approximated using a high-capacity network capable of generalization via interpolation such as a deep neural network.
We show a practical algorithm that optimizes this objective in Section~\ref{sec:practical}.

\subsection{Theoretical Analysis}

The goal of our analysis is to show that our approach SCQ enjoys the benefits of pessimism, while also not suffering from overly conservative behavior. We illustrate this by showing point-wise pessimism in the Q-function, but tighter pessimism, or not as strong of a lower-bound, compared to CQL.

Because our approach benefits from the interpolation of value function at actions near the data distribution, we must consider MDPs with continuous action spaces rather than a naive tabular one. For our analysis, we consider a linear MDP as defined in \citet{pevi}. However, note that our analysis can be extended to the general case of function approximation using analysis tools such as the neural tangent kernel (NTK) \citep{ntk}. 

\begin{dfn}
 A MDP $(\mathcal{S}, \mathcal{A}, P, r, p_0, \gamma)$ is a linear MDP if there exists known feature map $\phi: \mathcal{S} \times \mathcal{A} \to \mathbb{R}^d$ and $d$ unknown measures $\mu_h = (\mu_h^{(1)}, \mu_h^{(2)}, \hdots, \mu_h^{(d)})$ over $\mathcal{S}$ and vector $\theta \in \mathbb{R}^d$ such that
 \begin{align*}
     P(s' \mid s, a) = \langle \phi(s, a), \mu(s') \rangle\,, \quad r(s, a) = \langle \phi(s, a), \theta \rangle\,.
 \end{align*}
 for all $(s, a, s') \in \mathcal{S} \times \mathcal{A} \times \mathcal{S}$. Without loss of generality, we assume $||\phi(s,a)|| \leq 1$ for all $(s, a) \in \mathcal{S} \times \mathcal{A}$ and and $\max \{\|\mu(S)\|, \|\theta\|\} \leq \sqrt{d}$. Note that we define $\|\mu_h(S)\| = \int_S \|\mu_h(x)\| \, dx$.

\end{dfn}

Ultimately, we aim to show that while SCQ benefits from the theoretical guarantees due to pessimism just like CQL, SCQ produces tighter bounds.
Due to space, we simply state the results and deter proofs to Appendix~\ref{sec:appendix-proofs}.
Note that in the linear setting, learning at every iteration $k \in \mathbb{N}$ involves optimizing for weights $\hat{w}^k$ such that the Q-values and derived policy become 
\begin{align*}
    & \hat{Q}^k(s, a) = \langle \phi(s, a), \hat{w}^k \rangle\,, \quad 
    \hat{\pi}^k(\cdot \mid s) = \argmax_\pi \langle \hat{Q}^k(s, \cdot), \pi(\cdot \mid s) \rangle\,, 
\end{align*}
 Let us first consider CQL. We can show that in CQL, the learned value function satisfies the following pessimism guarantee:

\begin{lem}
For every iteration $k$, let $\hat{Q}^k, \hat{\pi}^k$ be the Q-function and policy learned by CQL in a linear MDP. Also, let $\hat{V}^k(s) = \langle \hat{\pi}^k(\cdot \mid s), \hat{Q}^k(s, \cdot) \rangle$, and $V^k$ be the exact value function in the absence of sampling error. Then, there exists $\alpha$ in the CQL objective in \eqref{eq:cql-q-update} such that $\hat{V}^k$ satisfies
$
    \hat{V}^k(s) \leq V^k(s)
$
for any state $s \in \gS$.
\end{lem}

The proof of this statement can be found in \citet{kumar2020conservative} and arises from solving for $\hat{w}^k$ that minimizes the CQL objective. Our main result is twofold. First, we can show that our approach SCQ satisfies point-wise pessimism with the following assumption.
\begin{ass} \label{ass:epsilon-bound}
For every state $s$, and action that not from the OOD actions set $\left(a \notin A_{\text{ood}}(s) \right)$, the sampling error between empirical $\hat{Q}$ and true $Q$ is bounded by $\varepsilon$.
\begin{equation}
    | Q(s,a) - \hat{Q}(s,a) | \leq \varepsilon 
\end{equation}
\end{ass}

\begin{thm} \label{thm:scq-pointwise}
For every iteration $k$, let $\hat{Q}^k$ be the Q-function learned by SCQ in a linear MDP and $Q^k$ be the true function without sampling error.  Then, there exists $\alpha$ in the SCQ objective in \eqref{eq:scq-q-update} such that $\hat{Q}^k$ satisfies
$
    \hat{Q}^k(s, a) \leq Q^k(s, a) + \varepsilon
$
for any state $s \in \gS$ and action $a \in \gA$.
\end{thm}

Second, while our approach satisfies learning lower bounds of the true value, the bounds are tighter than those learned by CQL due to our strategic conservatism. Namely, the value function learned by SCQ satisfies:

\begin{thm} \label{thm:scq-state-value}
For every iteration $k$, again let $\hat{V}^k(s)$ be the learned value function by SCQ, and $V^k$ be the exact one in the absence of sampling error. 
Similarly, let $\tilde{V}^k(s)$ be the value function learned by CQL.
If $\pi^k_{\text{ood}}$ satisfies $\pi^k_{\text{ood}}(\cdot | s) \neq \mathbf{0}$ for all states $s \in S$,
then, for any $\tilde{\alpha}$ in the CQL objective in \eqref{eq:cql-q-update}, there exists $\alpha$ in the SCQ objective in \eqref{eq:scq-q-update} such that $\hat{V}^k$ satisfies
$
    \tilde{V}^k(s) \leq 
    \hat{V}^k(s) \leq 
    V^k(s)
$
for any state $s \in \gS$.
\end{thm}

This means that our learned Q-values enjoy the provably efficient theoretical guarantees from being pessimistic estimations ~\citep{jin2020pessimism}. In addition, because our learned Q-values are tighter pessimistic estimates than those of CQL, our method also enjoys better sample efficiency.

\section{Practical Implementation}
\label{sec:practical}
In this section, we show how to optimize our novel objective in a practical algorithm. 
Recall that we assume that we are modeling the estimated value function and policy with deep neural network. Here, we use the same notation defined in Section\ref{subsec:rl-preliminaries}. The overall algorithm flow is described in Algorithm.\ref{alg:scq}.

\subsection{Approximately Identifying OOD Actions}
First, we show how to efficiently identify what actions are OOD actions for any state $s$. Note that our definition in Definition~\ref{def:ood-actions} is impractical to implement as it assumes iterating over the entire dataset.
Following previous approaches \cite{mcq}, we employ a Conditional Variational Autoencoder (CVAE) \cite{cvae} to classify OOD actions from the sampled data. The CVAE is trained on the entire dataset to model the behavior policy $\pi_{\beta}$. Specifically, the CVAE takes the state $s$ and action $a$ as inputs and outputs a reconstructed action $a_{\text{cvae}} = f_{\text{cvae}}(s,a)$.




When training our method SCQ, we use the learned policy to sample actions at a given state $s$. These sampled actions are then fed into the CVAE to generate reconstructed actions. 
By measuring the distance between the reconstructed actions and the sampled actions, we can infer whether the actions are OOD. 
Actions with a large reconstruction error are considered to be OOD, as they are not well-represented by the behavior policy modeled by the CVAE. Therefore, we approximate the set of OOD actions in the following way:
\begin{dfn} [Approximated out-of-distribution actions] \label{def:approximated-ood-actions}
    We define the approximate OOD action set $\hat{A}_{ood}(s)$ at a state $s$ as
    \begin{equation} \label{eq:approximated-ood-action-definition}
        \hat{A}_{ood}(s) = \{a \mid || a - f_{\text{cvae}}(s,a) ||_2 \geq \delta \}
    \end{equation}
    where $\delta$ is a threshold given by the user. 
\end{dfn}
In this paper, we calculate $\delta$ by taking the average distance over data in the dataset.
\begin{equation} \label{eq:dist-threshold}
    \delta = \mathbb{E}_{(s, a) \sim D} \big[||a - f_{cvae}(s,a) ||_2 \big]
\end{equation}
Using \eqref{eq:approximated-ood-action-definition} and \eqref{eq:dist-threshold}, we can explicitly distinguish OOD actions from in-distribution actions. Note that when datasets pose high multi-modal properties, complex generative models, such as diffusion models, can also be applied.

\subsection{Policy Evaluation Step}
It is difficult to perform the exact Q-update defined in \eqref{eq:scq-q-update} as we only have the approximate OOD action sets. 
Therefore, instead we solve the following optimization problem to update the Q-values at iteration $k$ parameterized by $\theta_k$: 
\begin{equation} \label{eq:scq-approximated-q-update}
\theta_{k+1} \leftarrow \underset{\theta}{\operatorname{argmin}} \; \alpha  \mathbb{E}_{s \sim \mathcal{D}, a \sim \hat{\pi}_{\text{ood}}} \big[ Q_{\theta}(s,a) \big] + \frac{1}{2} \mathbb{E}_{(s,a,r, s') \sim \mathcal{D}} \Bigg[ \big( Q_{\theta}(s,a) - \hat{B}^{\pi_{\phi_k}} Q_{\bar{\theta}_k}(s,a) \big)^2 \Bigg]\,,
\end{equation}
where $\hat{\pi}_{\text{ood}}$ constraints $\pi_{\phi_k}$ to only be in the set of approximate OOD actions $\hat{A}_{ood}(s)$ for state $s$. 
Here, the approximated Bellman update term $\hat{B}^{\pi_{\phi_k}} Q_{\bar{\theta}_k}(s,a)$ is expressed as:
\begin{align*}
\hat{B}^{\pi_{\phi_k}} Q_{\bar{\theta}_k}(s,a) = r + \gamma Q_{\bar{\theta}_k}(s', a')\,, \text{where} \qquad a' \sim \pi_{\phi_k}(\cdot \mid s')\,,
\end{align*}
and $\bar{\theta}_k$ are again target parameters that are updated more slowly than the parameters being optimized. 

\subsection{Policy Improvement Step}
Finally, we update the policy $\pi_{\phi_k}(\cdot|s)$. 
Our policy improvement step is similar to that of Soft Actor-Critic (SAC), which performs:
\begin{equation} 
\label{eq:scq-policy-update}
\phi_{k+1} \leftarrow \underset{\phi}{\operatorname{argmax}} \; \mathbb{E}_{(s, a_b) \sim \mathcal{D}, a \sim \pi_{\phi}(\cdot|s)} \left[ Q_{\theta_{k+1}}(s, a) - \lambda \log \pi_{\phi}(a|s) - \beta (a - a_b) \right] 
\end{equation}
where $\lambda\geq0$ is a hyperparameter to tune the entropy temperature and $\beta\geq0$ is a hyperparameter to ensure the learning policy follows the behavior policy $\pi_{b}$. When updating the policy parameter, we incorporate the behavior cloning term $\beta (a - a_{\pi_{b}})$, inspired by ReBRAC \cite{rebrac} and TD3-BC \cite{td3-bc}. Although this term is disabled in simpler environments such as Mujoco, we found it enhances SCQ's performance in more complex environments. Therefore, we only apply it during SCQ experiments with AntMaze.

\begin{figure}[ht]
\begin{algorithm}[H]
    \caption{Strategically Conservative Q-learning (SCQ)}
    \label{alg:scq}
    \begin{algorithmic}[1]    
    \REQUIRE Dataset $\mathcal{D}=\{(s,a, r, s')\}$, Maximum number of iterations $T$
    \STATE Initialize parameters $\theta, \bar{\theta}, \phi$
    \FOR{k=0 to T}
        \STATE Sample a mini-batch $B$ from the dataset $\mathcal{D}$
        \STATE Update CVAE to reconstrate $a$ from $(s, a) \sim B$
        \STATE For state $s \sim B$, sample OOD actions $a_{\text{ood}} \sim \hat{\pi}_{\text{ood}}(\cdot | s)$ via rejection sampling on $\pi_{\phi}$ using the CVAE and Equation~\ref{eq:approximated-ood-action-definition}
        \STATE Update Q-value parameters $\theta$ using Equation~\ref{eq:scq-approximated-q-update}
        \STATE Update policy parameters $\phi$ using Equation~\ref{eq:scq-policy-update}
        \STATE Update target parameters $\bar{\theta} \leftarrow (1 -\upsilon) \bar{\theta} + \upsilon \theta$ where $0\leq\upsilon\leq1$
    \ENDFOR
    \end{algorithmic}
\end{algorithm}
\end{figure}

\section{Experiments} \label{sec:experiments}
In this section, we evaluate the efficacy of SCQ by conducting experiments on the D4RL dataset \cite{d4rl2021}. The primary objective of our experimental framework is to benchmark our method against previous offline RL algorithms, with a specific emphasis on comparing it to other distance-sensitive offline reinforcement learning approaches.

\subsection{Evaluation on Mujoco and Antmaze} \label{sec:exp-mujoco-antmaze}
We compare SCQ with a range of recent model-free baseline algorithms such as CQL\cite{cql}, IQL\cite{iql}, TD3-BC\cite{td3-bc}, DOGE\cite{doge}, MCQ\cite{mcq} and SAC-RND\cite{rnd-sac} using Gym-Mujoco and Antmaze tasks. Mujoco environment has dense rewards, and it is easier for these model-free approaches to extract optimal policies from the static dataset. In contrast, Antmaze tasks present a considerable challenge due to their sparse rewards. Moreover, they contain less optimal or near-optimal trajectories in the dataset than Mujoco and therefore necessitate the "stitching" together of suboptimal trajectory segments to navigate from start to goal within the maze\cite{iql, d4rl2021}. We train SCQ 1 million steps for each task. Moreover, unlike \cite{td3-bc}, SCQ does not require state normalization. Note that, we do not compare MCQ with SCQ as MCQ does not provide its official scores for Antmaze tasks\footnote{We ran official MCQ code for Antmaze tasks by tuning parameters, but the final score is 0 for most of the tasks}. Comprehensive details on the experimental setup and methodologies are provided in the Appendix.\ref{sec:appendix-experimental-details}. 

Tables \ref{tab:mujoco-result} and \ref{tab:antmaze-result} summarize the average scores and associated standard deviations for the Mujoco and Antmaze tasks, respectively. We derived the CQL scores from \cite{ghasemipour2022why}, the TD3-BC, IQL, and SAC-RND scores from \cite{rebrac}, and the DOGE scores from the original paper\cite{doge}. It is pertinent to note that DOGE does not provide scores for the `-expert` dataset; therefore, we run their official code using 10 random seeds. Since we were unable to reproduce official MCQ scores, we rerun their official code with 10 random seeds.

These findings demonstrate that SCQ achieves superior, state-of-the-art performance. Particularly in Antmaze tasks, SCQ consistently outperforms baseline methods by a significant margin. By comparing MCQ and SCQ,  we can tell the explicit separation of out-of-distribution data and the minimization of out-of-distribution Q-values are critical, indeed indispensable, for achieving exceptional performance.

\begin{table}[h]
  \caption{Comparison of normalized average scores for SCQ against baseline methods on Gym-MuJoCo tasks, based on the final 10 evaluations. These experiments were conducted using MuJoCo "-v2" datasets across 10 random seeds. The labels are defined as follows: r = random, m = medium, m-r = medium-replay, m-e = medium-expert, e = expert. The highest mean score is highlighted in bold, and the number following the symbol $\pm$ indicates the standard deviation across different seeds.}
  \label{tab:mujoco-result}
  \centering
  \footnotesize
  \scalebox{0.86}[0.93]{
  \begin{tabular}{llllllll}
    \toprule
    {Task Name} & {TD3+BC} & {IQL} & {CQL} & {DOGE} & {MCQ} & {SAC-RND} & {SCQ} \\
    \midrule
    halfcheetah-r   & 30.9 $\pm$ 0.4 & 19.5 $\pm$ 0.8 & \textbf{31.1 $\pm$ 3.5} & 17.8 $\pm$ 1.2 & 25.6 $\pm$ 1.4 & 29.5 $\pm$ 1.5  & 26.9 $\pm$ 0.9 \\
    hopper-r        & 8.5  $\pm$ 0.7 & 10.1 $\pm$ 5.9 & 5.3  $\pm$ 0.6 & 21.1 $\pm$ 12.6 & 12.7 $\pm$ 0.9 & 8.1 $\pm$ 2.4 & \textbf{31.9 $\pm$ 0.7} \\
    walker2d-r      & 2.7  $\pm$ 3.6 & 11.3 $\pm$ 7.0 & 7.5  $\pm$ 1.7 & 0.9 $\pm$ 2.4  & 5.9 $\pm$ 5.1 & \textbf{18.4 $\pm$ 4.5} & 4.75 $\pm$ 6.2 \\
    halfcheetah-m   & 54.7 $\pm$ 0.9 & 50.0 $\pm$ 0.8 & 45.3 $\pm$ 0.7 & 45.3 $\pm$ 0.6 & 54.4 $\pm$ 1.8 & 65.6 $\pm$ 1.0 & \textbf{68.3 $\pm$ 1.6} \\
    hopper-m        & 60.9 $\pm$ 7.6 & 65.2 $\pm$ 4.2 & 64.0 $\pm$ 0.8 & 98.6 $\pm$ 2.1 & 78.4 $\pm$ 4.3 & \textbf{102.0 $\pm$ 1.0} & 89.5 $\pm$ 5.2  \\
    walker2d-m      & 77.0 $\pm$ 2.9 & 80.7 $\pm$ 3.4 & 79.5 $\pm$ 1.2 & 86.8 $\pm$ 0.8 & 78.6 $\pm$ 1.7 & 82.5 $\pm$ 3.6 &\textbf{86.9 $\pm$ 0.6}  \\
    halfcheetah-m-r & 45.0 $\pm$ 1.1 & 47.2 $\pm$ 3.6 & 45.3 $\pm$ 0.9 & 42.8 $\pm$ 0.6 & 42.2 $\pm$ 0.7  & 51.0 $\pm$ 0.8 & \textbf{57.2 $\pm$ 1.8} \\
    hopper-m-r      & 55.1 $\pm$ 31.7& 89.6 $\pm$ 13.2& 86.3 $\pm$ 7.3 & 76.2 $\pm$ 17.7 & 88.6 $\pm$ 10.2 & 98.1 $\pm$ 5.3 & \textbf{101.4 $\pm$ 1.9} \\
    walker2d-m-r    & 68.0 $\pm$ 16.2& 75.4 $\pm$ 9.5 & 76.8 $\pm$ 0.0  & 87.3 $\pm$ 2.3  & 83.7 $\pm$ 9.0 & 77.3 $\pm$ 7.9 & \textbf{87.9 $\pm$ 4.7}  \\
    halfcheetah-m-e & 89.1 $\pm$ 5.6 & 92.7 $\pm$ 2.8 & 95.0 $\pm$ 1.4 & 78.7 $\pm$ 8.4 & 91.8 $\pm$ 0.9 & \textbf{101.1 $\pm$ 5.2}  & 98.8 $\pm$ 0.6 \\
    hopper-m-e      & 87.8 $\pm$ 10.5& 85.5 $\pm$ 29.7& 96.9 $\pm$ 15.1 & 102.7 $\pm$ 5.2  & 86.9 $\pm$ 21.3 & 107.0 $\pm$ 6.4 & \textbf{110.6 $\pm$ 1.9}  \\
    walker2d-m-e    & 110.0$\pm$ 0.6 & 96.9 $\pm$ 32.3& 109.3$\pm$ 0.3  & 110.4 $\pm$ 1.5 & 105.2 $\pm$ 1.6 & 111.6 $\pm$ 0.3 & \textbf{112.2 $\pm$ 0.5}  \\
    halfcheetah-e   & 93.4 $\pm$ 0.4 & 95.5 $\pm$ 2.1 & 97.3 $\pm$ 1.1 & 88.2 $\pm$ 2.1 & 93.8 $\pm$ 0.3  & \textbf{105.9 $\pm$ 1.7} & 103.0 $\pm$ 2.4 \\
    hopper-e        & 109.6$\pm$ 3.7 & 108.8$\pm$ 3.1 & 106.5$\pm$ 9.1 & 104.6 $\pm$ 5.6 & 95.0 $\pm$ 1.0 & 100.1 $\pm$ 8.3 & \textbf{111.3 $\pm$ 2.6} \\
    walker2d-e      & 110.0$\pm$ 0.6 & 96.9 $\pm$ 32.3& 109.3$\pm$ 0.3 & 109.3 $\pm$ 0.7 & 107.0 $\pm$ 0.6 & 112.3 $\pm$ 0.2 & \textbf{114.4 $\pm$ 1.3}  \\
    \midrule
    average & 66.8 & 68.4 & 70.4 & 71.4 & 70.0 & 78.0 &  \textbf{80.3} \\
    \bottomrule
  \end{tabular}
  }
\end{table}

\begin{table}[h]
  \caption{Comparison of normalized average scores for SCQ against baseline methods on Antmaze tasks, based on the final 100 evaluations. We use Antmaze "-v2" datasets across 10 random seeds to evaluate each algorithm. The labels are defined as follows: u = umaze, u-d = umaze-diverse, m-p = medium-play, m-d = medium-diverse, l-p = large-play, l-d = large-diverse. The highest mean score is highlighted in bold, and the number following the symbol $\pm$ indicates the standard deviation across different seeds.}
  \label{tab:antmaze-result}
  \centering
  \begin{tabular}{lllllll}
    \toprule
    {Task Name} & {TD3+BC} & {IQL} & {CQL} & {DOGE} & {SAC-RND} & {SCQ} \\
    \midrule
    antmaze-u   & 66.3 $\pm$ 6.2 & 83.3 $\pm$ 4.5 & 74.0 & 97.0 $\pm$ 1.8 & 97.0 $\pm$ 1.5 & \textbf{97.8 $\pm$ 1.1}  \\
    antmaze-u-d & 53.8 $\pm$ 8.5 & 70.6 $\pm$ 3.7 & 84.0 & 63.5 $\pm$ 9.3 & 66.0 $\pm$ 25.0 & \textbf{89.5 $\pm$ 9.3}  \\
    antmaze-m-p & 26.5 $\pm$ 18.4& 64.6 $\pm$ 4.9 & 61.2 & 80.6 $\pm$ 6.5 & 38.5 $\pm$ 29.4 & \textbf{81.1 $\pm$ 16.2} \\
    antmaze-m-d & 25.9 $\pm$ 15.3& 61.7 $\pm$ 6.1 & 53.7 & 77.6 $\pm$ 6.1 & 74.7 $\pm$ 10.7 & \textbf{79.4 $\pm$ 10.1} \\
    antmaze-l-p & 0.0  $\pm$ 0.0 & 42.5 $\pm$ 6.5 & 15.8 & 48.2 $\pm$ 8.1 & 43.9 $\pm$ 29.2 & \textbf{67.2 $\pm$ 10.5} \\
    antmaze-l-d & 0.0  $\pm$ 0.0 & 27.6 $\pm$ 7.8 & 14.9 & 36.4 $\pm$ 9.1 & 45.7 $\pm$ 28.5 & \textbf{60.0 $\pm$ 17.6} \\
    \midrule
    average & 28.7 & 58.3 & 50.6 & 67.2 & 60.9 & \textbf{79.2} \\
    \bottomrule
  \end{tabular}
\end{table}

\subsection{Ablation study} \label{sec:exp-ablation-study}
Exploring parameter sensitivity is a critical component of algorithm evaluation. In this study, we perform ablation experiments to assess the impact of varying the hyperparameter $\alpha$ on HalfCheetah and Antmaze tasks. We compare the baseline parameter set against two modifications: the first involves setting $\alpha=0$ to evaluate the effect of eliminating Q constraints, while the second replaces Q constraints with Layer Normalization (LN). Previous studies suggest that Layer Normalization effectively prevents the overestimation of Q-values \cite{rebrac, kumar2023offline, LNExplanation}. Throughout the experiments, we only change the parameter $\alpha$, maintaining default values for all others.

We report the normalized score of each method in Table \ref{tab:ablation}. First, we can see that setting $alpha=0$ significantly impacts performance, and the performance deteriorates drastically in most tasks. This indicates that Q minimization is critical for our method. Moreover, we can find that SCQ has higher scores than Layer Normalization in most cases. One possible reason is that Layer Normalization imposes excessive constraints on Q-values, resulting in inaccurate estimations. In the Appendix.\ref{sec:appendix-additional-ablation-studies}, we provide additional ablation results that explore variations in the parameter $\alpha$.

\begin{table}[ht]
  \caption{Ablation results of SCQ on Gym-Mujoco and Antmaze tasks. We evaluated each method using the "-v2" datasets across 10 random seeds. The highest mean score is highlighted in bold, and the number following the symbol $\pm$ indicates the standard deviation across different seeds.}
  \label{tab:ablation}
  \centering
  \begin{tabular}{llll}
    \toprule
    {Task Name} & {SAC($\alpha$ = 0)} & {LN} & {SCQ} \\
    \midrule
    halfcheetah-medium        & 49.6 $\pm$ 10.3 & 62.4 $\pm$ 1.4 & \bf{68.3} $\pm$ 1.6  \\
    halfcheetah-medium-replay & 39.3 $\pm$ 12.6 & 45.9 $\pm$ 4.2 & \bf{57.2} $\pm$ 1.8  \\
    halfcheetah-medium-expert & 13.4 $\pm$ 9.0  & 48.6 $\pm$ 1.4 & \bf{98.8} $\pm$ 0.6  \\
    antmaze-medium-play       & 64.0 $\pm$ 31.7 & 74.1 $\pm$ 8.8 & \bf{81.1} $\pm$ 16.2 \\
    antmaze-medium-diverse    & 70.9 $\pm$ 28.3 & 50.6 $\pm$ 21.7 & \bf{79.4} $\pm$ 10.1 \\
    antmaze-large-play        & 61.6 $\pm$ 33.1 & \bf{72.4} $\pm$ 21.3 & 67.2 $\pm$ 10.5 \\
    antmaze-large-diverse     & \bf{62.6} $\pm$ 14.8 & 45.6 $\pm$ 31.3 & 60.0 $\pm$ 17.6 \\
    \bottomrule
  \end{tabular}
\end{table}

\subsection{X\% Offline Dataset}
It is one of the most challenging situations for offline reinforcement learning to obtain an optimal policy from a small dataset. Insufficient data can harm the accuracy of estimated Q-values, potentially leading to divergence in the worst case. Layer Normalization has been shown to exhibit robustness to dataset size variations \cite{LNExplanation}. In this experiment, we compare SCQ with Layer Normalization in terms of the robustness of the data size. We systematically reduce the size of the D4RL dataset by randomly selecting samples from the original dataset, scaling it down to 50\%, 30\%, and 10\% of its original size. Note that we use the same parameters as described in Section\ref{sec:exp-mujoco-antmaze} throughout these tests.

The results are shown in Fig.\ref{fig:exp-small-datasize}. We can find that SCQ typically outperforms Layer Normalization even when the dataset size decreases. This is attributed to the tendency of Layer Normalization to be overly conservative, often leading to an underestimation of Q-values, which is also discussed in Section \ref{sec:exp-ablation-study}.

\begin{figure}[hbtp]
  \centering
  \includegraphics[scale=0.45]{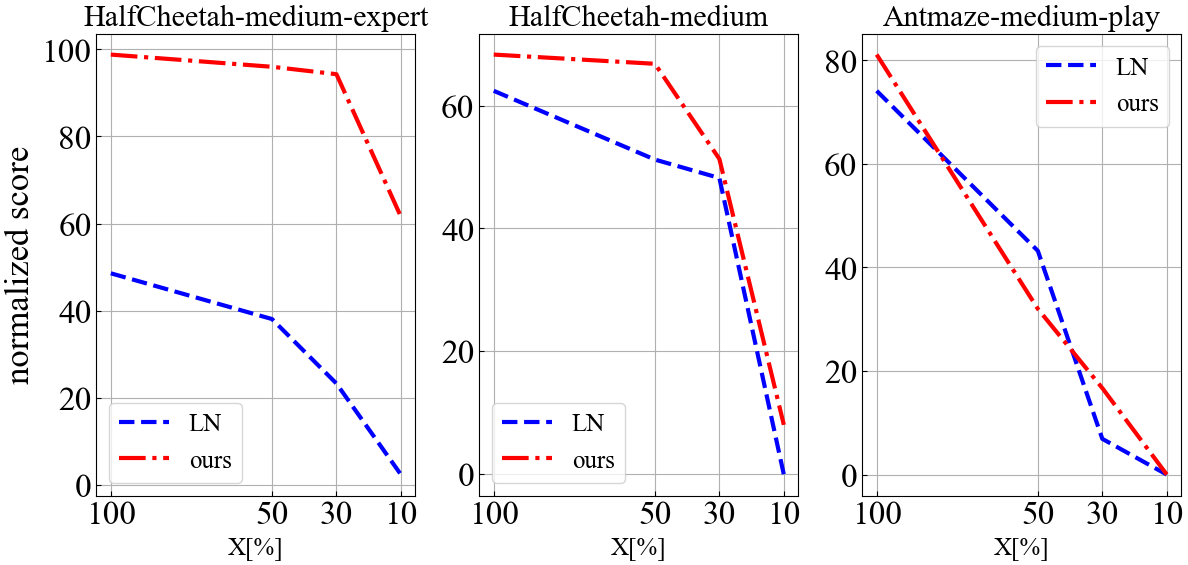}
  \caption{Comparison of normalized average scores for SCQ against Layer Normalization (LN) on D4RL X\% dataset. We evaluated each method using the "-v2" datasets across 10 random seeds.}
  \label{fig:exp-small-datasize}
\end{figure}

\section{Conclusion} \label{sec:conclusion}
In this work, we design a new offline reinforcement algorithm, called Strategically Conservative Q-learning (SCQ), which is designed to exploit the interpolation capabilities of deep neural networks. This new approach leverages the neural network interpolation ability by explicitly minimizing the Q-function in regions that are far from the data distribution. We prove that the estimated Q function obtained by the proposed method is the point-wise lower bound of the true Q function. Experimental evaluations using the D4RL dataset illustrate that our algorithm outperforms baseline methods by a considerable margin and exhibits robustness in the face of reduced dataset sizes.

\paragraph{Limitations.} Despite the theoretical foundations and experimental results supporting SCQ, it has two notable drawbacks. First, the CVAE may make errors in distinguishing out-of-distribution (ODD) actions from in-distribution (ID) actions, potentially misclassifying ID actions as OOD actions. Second, SCQ necessitates extensive parameter tuning across different environments. Future research will thus aim to develop a method for automatic parameter adjustment.

\newpage
\bibliographystyle{abbrvnat}
\bibliography{bibliography}

\begin{thebibliography}{54}
\providecommand{\natexlab}[1]{#1}
\providecommand{\url}[1]{\texttt{#1}}
\expandafter\ifx\csname urlstyle\endcsname\relax
  \providecommand{\doi}[1]{doi: #1}\else
  \providecommand{\doi}{doi: \begingroup \urlstyle{rm}\Url}\fi

\bibitem[Arora et~al.(2019)Arora, Du, Hu, Li, and Wang]{Arora2019FineGrained}
S.~Arora, S.~Du, W.~Hu, Z.~Li, and R.~Wang.
\newblock Fine-grained analysis of optimization and generalization for overparameterized two-layer neural networks.
\newblock In \emph{36th International Conference on Machine Learning, ICML 2019}, 36th International Conference on Machine Learning, ICML 2019, pages 477--502. International Machine Learning Society (IMLS), Jan. 2019.
\newblock 36th International Conference on Machine Learning, ICML 2019 ; Conference date: 09-06-2019 Through 15-06-2019.

\bibitem[Atkeson and Schaal(1997)]{Atkeson1997RobotLearning}
C.~G. Atkeson and S.~Schaal.
\newblock Robot learning from demonstration.
\newblock In \emph{Proceedings of the Fourteenth International Conference on Machine Learning}, ICML '97, page 12–20, San Francisco, CA, USA, 1997. Morgan Kaufmann Publishers Inc.
\newblock ISBN 1558604863.

\bibitem[Barnard and Wessels(1992)]{Barnard1992Extrapolation}
E.~Barnard and L.~Wessels.
\newblock Extrapolation and interpolation in neural network classifiers.
\newblock \emph{IEEE Control Systems Magazine}, 12\penalty0 (5):\penalty0 50--53, 1992.
\newblock \doi{10.1109/37.158898}.

\bibitem[Chen et~al.(2020)Chen, Zhou, Wang, Wang, Wu, Deng, and Ross]{chen2020bail}
X.~Chen, Z.~Zhou, Z.~Wang, C.~Wang, Y.~Wu, Q.~Deng, and K.~Ross.
\newblock {\{}BAIL{\}}: Best-action imitation learning for batch deep reinforcement learning, 2020.
\newblock URL \url{https://openreview.net/forum?id=BJlnmgrFvS}.

\bibitem[Dadashi et~al.(2021)Dadashi, Rezaeifar, Vieillard, Hussenot, Pietquin, and Geist]{Dadashi2021Offline}
R.~Dadashi, S.~Rezaeifar, N.~Vieillard, L.~Hussenot, O.~Pietquin, and M.~Geist.
\newblock Offline reinforcement learning with pseudometric learning.
\newblock In M.~Meila and T.~Zhang, editors, \emph{Proceedings of the 38th International Conference on Machine Learning}, volume 139 of \emph{Proceedings of Machine Learning Research}, pages 2307--2318. PMLR, 18--24 Jul 2021.
\newblock URL \url{https://proceedings.mlr.press/v139/dadashi21a.html}.

\bibitem[Florence et~al.(2022)Florence, Lynch, Zeng, Ramirez, Wahid, Downs, Wong, Lee, Mordatch, and Tompson]{Pete2022ImplicitBehaviorCloning}
P.~Florence, C.~Lynch, A.~Zeng, O.~A. Ramirez, A.~Wahid, L.~Downs, A.~Wong, J.~Lee, I.~Mordatch, and J.~Tompson.
\newblock Implicit behavioral cloning.
\newblock In A.~Faust, D.~Hsu, and G.~Neumann, editors, \emph{Proceedings of the 5th Conference on Robot Learning}, volume 164 of \emph{Proceedings of Machine Learning Research}, pages 158--168. PMLR, 08--11 Nov 2022.
\newblock URL \url{https://proceedings.mlr.press/v164/florence22a.html}.

\bibitem[Fu et~al.(2021)Fu, Kumar, Nachum, Tucker, and Levine]{d4rl2021}
J.~Fu, A.~Kumar, O.~Nachum, G.~Tucker, and S.~Levine.
\newblock D4rl: Datasets for deep data-driven reinforcement learning, 2021.
\newblock URL \url{https://openreview.net/forum?id=px0-N3_KjA}.

\bibitem[Fujimoto and Gu(2021)]{td3-bc}
S.~Fujimoto and S.~Gu.
\newblock A minimalist approach to offline reinforcement learning.
\newblock In A.~Beygelzimer, Y.~Dauphin, P.~Liang, and J.~W. Vaughan, editors, \emph{Advances in Neural Information Processing Systems}, 2021.
\newblock URL \url{https://openreview.net/forum?id=Q32U7dzWXpc}.

\bibitem[Fujimoto et~al.(2018{\natexlab{a}})Fujimoto, Meger, and Precup]{Fujimoto2018OffPolicyDR}
S.~Fujimoto, D.~Meger, and D.~Precup.
\newblock Off-policy deep reinforcement learning without exploration.
\newblock In \emph{International Conference on Machine Learning}, 2018{\natexlab{a}}.
\newblock URL \url{https://api.semanticscholar.org/CorpusID:54457299}.

\bibitem[Fujimoto et~al.(2018{\natexlab{b}})Fujimoto, van Hoof, and Meger]{fujimoto2018addressing}
S.~Fujimoto, H.~van Hoof, and D.~Meger.
\newblock Addressing function approximation error in actor-critic methods.
\newblock In \emph{International Conference on Machine Learning (ICML)}, pages 1587--1596, 2018{\natexlab{b}}.

\bibitem[Fujimoto et~al.(2018{\natexlab{c}})Fujimoto, van Hoof, and Meger]{td3}
S.~Fujimoto, H.~van Hoof, and D.~Meger.
\newblock Addressing function approximation error in actor-critic methods.
\newblock In \emph{International Conference on Machine Learning}, 2018{\natexlab{c}}.
\newblock URL \url{https://api.semanticscholar.org/CorpusID:3544558}.

\bibitem[Ghasemipour et~al.(2022)Ghasemipour, Gu, and Nachum]{ghasemipour2022why}
S.~K.~S. Ghasemipour, S.~S. Gu, and O.~Nachum.
\newblock Why so pessimistic? estimating uncertainties for offline {RL} through ensembles, and why their independence matters., 2022.
\newblock URL \url{https://openreview.net/forum?id=wQ7RCayXUSl}.

\bibitem[Haarnoja et~al.(2018{\natexlab{a}})Haarnoja, Zhou, Abbeel, and Levine]{sac}
T.~Haarnoja, A.~Zhou, P.~Abbeel, and S.~Levine.
\newblock Soft actor-critic: Off-policy maximum entropy deep reinforcement learning with a stochastic actor.
\newblock In J.~Dy and A.~Krause, editors, \emph{Proceedings of the 35th International Conference on Machine Learning}, volume~80 of \emph{Proceedings of Machine Learning Research}, pages 1861--1870. PMLR, 10--15 Jul 2018{\natexlab{a}}.
\newblock URL \url{https://proceedings.mlr.press/v80/haarnoja18b.html}.

\bibitem[Haarnoja et~al.(2018{\natexlab{b}})Haarnoja, Zhou, Hartikainen, Tucker, Ha, Tan, Kumar, Zhu, Gupta, Abbeel, and Levine]{sac2}
T.~Haarnoja, A.~Zhou, K.~Hartikainen, G.~Tucker, S.~Ha, J.~Tan, V.~Kumar, H.~Zhu, A.~Gupta, P.~Abbeel, and S.~Levine.
\newblock Soft actor-critic algorithms and applications.
\newblock \emph{CoRR}, abs/1812.05905, 2018{\natexlab{b}}.
\newblock URL \url{http://arxiv.org/abs/1812.05905}.

\bibitem[Haarnoja et~al.(2019)Haarnoja, Ha, Zhou, Tan, Tucker, and Levine]{sac3}
T.~Haarnoja, S.~Ha, A.~Zhou, J.~Tan, G.~Tucker, and S.~Levine.
\newblock Learning to walk via deep reinforcement learning.
\newblock In \emph{Proceedings of Robotics: Science and Systems}, FreiburgimBreisgau, Germany, June 2019.
\newblock \doi{10.15607/RSS.2019.XV.011}.

\bibitem[Haley and Soloway(1992)]{Haley1992Extrapolation}
P.~Haley and D.~Soloway.
\newblock Extrapolation limitations of multilayer feedforward neural networks.
\newblock In \emph{[Proceedings 1992] IJCNN International Joint Conference on Neural Networks}, volume~4, pages 25--30 vol.4, 1992.
\newblock \doi{10.1109/IJCNN.1992.227294}.

\bibitem[Jacot et~al.(2018)Jacot, Gabriel, and Hongler]{ntk}
A.~Jacot, F.~Gabriel, and C.~Hongler.
\newblock Neural tangent kernel: Convergence and generalization in neural networks.
\newblock In S.~Bengio, H.~Wallach, H.~Larochelle, K.~Grauman, N.~Cesa-Bianchi, and R.~Garnett, editors, \emph{Advances in Neural Information Processing Systems}, volume~31. Curran Associates, Inc., 2018.
\newblock URL \url{https://proceedings.neurips.cc/paper_files/paper/2018/file/5a4be1fa34e62bb8a6ec6b91d2462f5a-Paper.pdf}.

\bibitem[Jin et~al.(2020)Jin, Yang, and Wang]{pevi}
Y.~Jin, Z.~Yang, and Z.~Wang.
\newblock Is pessimism provably efficient for offline rl?
\newblock \emph{CoRR}, abs/2012.15085, 2020.

\bibitem[Jin et~al.(2021)Jin, Yang, and Wang]{jin2020pessimism}
Y.~Jin, Z.~Yang, and Z.~Wang.
\newblock Is pessimism provably efficient for offline rl?
\newblock In \emph{International Conference on Machine Learning}. PMLR, 2021.

\bibitem[Kingma and Ba(2014)]{adam}
D.~Kingma and J.~Ba.
\newblock Adam: A method for stochastic optimization.
\newblock \emph{International Conference on Learning Representations}, 12 2014.

\bibitem[Kostrikov et~al.(2021)Kostrikov, Fergus, Tompson, and Nachum]{Ilya2021Fisher}
I.~Kostrikov, R.~Fergus, J.~Tompson, and O.~Nachum.
\newblock Offline reinforcement learning with fisher divergence critic regularization.
\newblock In M.~Meila and T.~Zhang, editors, \emph{ICML}, volume 139 of \emph{Proceedings of Machine Learning Research}, pages 5774--5783. PMLR, 2021.
\newblock URL \url{http://dblp.uni-trier.de/db/conf/icml/icml2021.html#KostrikovFTN21}.

\bibitem[Kostrikov et~al.(2022)Kostrikov, Nair, and Levine]{iql}
I.~Kostrikov, A.~Nair, and S.~Levine.
\newblock Offline reinforcement learning with implicit q-learning.
\newblock In \emph{International Conference on Learning Representations}, 2022.
\newblock URL \url{https://openreview.net/forum?id=68n2s9ZJWF8}.

\bibitem[Kumar et~al.(2019{\natexlab{a}})Kumar, Fu, Soh, Tucker, and Levine]{Kumar2019Stabilizing}
A.~Kumar, J.~Fu, M.~Soh, G.~Tucker, and S.~Levine.
\newblock Stabilizing off-policy q-learning via bootstrapping error reduction.
\newblock In H.~Wallach, H.~Larochelle, A.~Beygelzimer, F.~d\textquotesingle Alch\'{e}-Buc, E.~Fox, and R.~Garnett, editors, \emph{Advances in Neural Information Processing Systems}, volume~32. Curran Associates, Inc., 2019{\natexlab{a}}.
\newblock URL \url{https://proceedings.neurips.cc/paper_files/paper/2019/file/c2073ffa77b5357a498057413bb09d3a-Paper.pdf}.

\bibitem[Kumar et~al.(2019{\natexlab{b}})Kumar, Fu, Soh, Tucker, and Levine]{bear}
A.~Kumar, J.~Fu, M.~Soh, G.~Tucker, and S.~Levine.
\newblock Stabilizing off-policy q-learning via bootstrapping error reduction.
\newblock In \emph{Neural Information Processing Systems (NeurIPS)}, 2019{\natexlab{b}}.

\bibitem[Kumar et~al.(2020{\natexlab{a}})Kumar, Zhou, Tucker, and Levine]{cql}
A.~Kumar, A.~Zhou, G.~Tucker, and S.~Levine.
\newblock Conservative q-learning for offline reinforcement learning.
\newblock In H.~Larochelle, M.~Ranzato, R.~Hadsell, M.~Balcan, and H.~Lin, editors, \emph{Advances in Neural Information Processing Systems}, volume~33, pages 1179--1191. Curran Associates, Inc., 2020{\natexlab{a}}.
\newblock URL \url{https://proceedings.neurips.cc/paper_files/paper/2020/file/0d2b2061826a5df3221116a5085a6052-Paper.pdf}.

\bibitem[Kumar et~al.(2020{\natexlab{b}})Kumar, Zhou, Tucker, and Levine]{kumar2020conservative}
A.~Kumar, A.~Zhou, G.~Tucker, and S.~Levine.
\newblock Conservative q-learning for offline reinforcement learning.
\newblock \emph{arXiv preprint arXiv:2006.04779}, 2020{\natexlab{b}}.

\bibitem[Kumar et~al.(2023)Kumar, Agarwal, Geng, Tucker, and Levine]{kumar2023offline}
A.~Kumar, R.~Agarwal, X.~Geng, G.~Tucker, and S.~Levine.
\newblock Offline q-learning on diverse multi-task data both scales and generalizes.
\newblock In \emph{The Eleventh International Conference on Learning Representations}, 2023.
\newblock URL \url{https://openreview.net/forum?id=4-k7kUavAj}.

\bibitem[Lagoudakis and Parr(2003)]{lstdq}
M.~G. Lagoudakis and R.~Parr.
\newblock Least-squares policy iteration.
\newblock \emph{J. Mach. Learn. Res.}, 4\penalty0 (null):\penalty0 1107–1149, dec 2003.
\newblock ISSN 1532-4435.

\bibitem[Laskey et~al.(2017)Laskey, Lee, Hsieh, Liaw, Mahler, Fox, and Goldberg]{dart}
M.~Laskey, J.~Lee, W.~Y. Hsieh, R.~Liaw, J.~Mahler, R.~Fox, and K.~Goldberg.
\newblock Iterative noise injection for scalable imitation learning.
\newblock \emph{CoRR}, abs/1703.09327, 2017.
\newblock URL \url{http://arxiv.org/abs/1703.09327}.

\bibitem[Le et~al.(2019)Le, Voloshin, and Yue]{Le2019BatchPL}
H.~M. Le, C.~Voloshin, and Y.~Yue.
\newblock Batch policy learning under constraints.
\newblock In \emph{International Conference on Machine Learning}, 2019.
\newblock URL \url{https://api.semanticscholar.org/CorpusID:84842462}.

\bibitem[Lee et~al.(2021)Lee, Seo, Lee, Abbeel, and Shin]{lee2021offlinetoonline}
S.~Lee, Y.~Seo, K.~Lee, P.~Abbeel, and J.~Shin.
\newblock Offline-to-online reinforcement learning via balanced replay and pessimistic q-ensemble.
\newblock In \emph{5th Annual Conference on Robot Learning}, 2021.
\newblock URL \url{https://openreview.net/forum?id=AlJXhEI6J5W}.

\bibitem[Levine et~al.(2020)Levine, Kumar, Tucker, and Fu]{levine2020offline}
S.~Levine, A.~Kumar, G.~Tucker, and J.~Fu.
\newblock Offline reinforcement learning: Tutorial, review, and perspectives on open problems, 2020.

\bibitem[Li et~al.(2023)Li, Zhan, Xu, Zhu, Liu, and Zhang]{doge}
J.~Li, X.~Zhan, H.~Xu, X.~Zhu, J.~Liu, and Y.-Q. Zhang.
\newblock When data geometry meets deep function: Generalizing offline reinforcement learning.
\newblock In \emph{The Eleventh International Conference on Learning Representations}, 2023.
\newblock URL \url{https://openreview.net/forum?id=lMO7TC7cuuh}.

\bibitem[Loshchilov and Hutter(2017)]{sgdr}
I.~Loshchilov and F.~Hutter.
\newblock {SGDR}: Stochastic gradient descent with warm restarts.
\newblock In \emph{International Conference on Learning Representations}, 2017.
\newblock URL \url{https://openreview.net/forum?id=Skq89Scxx}.

\bibitem[Lyu et~al.(2022)Lyu, Ma, Li, and Lu]{mcq}
J.~Lyu, X.~Ma, X.~Li, and Z.~Lu.
\newblock Mildly conservative q-learning for offline reinforcement learning.
\newblock In A.~H. Oh, A.~Agarwal, D.~Belgrave, and K.~Cho, editors, \emph{Advances in Neural Information Processing Systems}, 2022.
\newblock URL \url{https://openreview.net/forum?id=VYYf6S67pQc}.

\bibitem[Matsushima et~al.(2021)Matsushima, Furuta, Matsuo, Nachum, and Gu]{matsushima2021deploymentefficient}
T.~Matsushima, H.~Furuta, Y.~Matsuo, O.~Nachum, and S.~Gu.
\newblock Deployment-efficient reinforcement learning via model-based offline optimization.
\newblock In \emph{International Conference on Learning Representations}, 2021.
\newblock URL \url{https://openreview.net/forum?id=3hGNqpI4WS}.

\bibitem[Nair et~al.(2021)Nair, Dalal, Gupta, and Levine]{awac}
A.~Nair, M.~Dalal, A.~Gupta, and S.~Levine.
\newblock {\{}AWAC{\}}: Accelerating online reinforcement learning with offline datasets, 2021.
\newblock URL \url{https://openreview.net/forum?id=OJiM1R3jAtZ}.

\bibitem[Nakamoto et~al.(2023)Nakamoto, Zhai, Singh, Mark, Ma, Finn, Kumar, and Levine]{calql}
M.~Nakamoto, Y.~Zhai, A.~Singh, M.~S. Mark, Y.~Ma, C.~Finn, A.~Kumar, and S.~Levine.
\newblock Cal-{QL}: Calibrated offline {RL} pre-training for efficient online fine-tuning.
\newblock In \emph{Thirty-seventh Conference on Neural Information Processing Systems}, 2023.
\newblock URL \url{https://openreview.net/forum?id=GcEIvidYSw}.

\bibitem[Nikulin et~al.(2023)Nikulin, Kurenkov, Tarasov, and Kolesnikov]{rnd-sac}
A.~Nikulin, V.~Kurenkov, D.~Tarasov, and S.~Kolesnikov.
\newblock Anti-exploration by random network distillation.
\newblock In \emph{Proceedings of the 40th International Conference on Machine Learning}, ICML'23. JMLR.org, 2023.

\bibitem[Pastor et~al.(2009)Pastor, Hoffmann, Asfour, and Schaal]{Pastor2009Learning}
P.~Pastor, H.~Hoffmann, T.~Asfour, and S.~Schaal.
\newblock Learning and generalization of motor skills by learning from demonstration.
\newblock In \emph{2009 IEEE International Conference on Robotics and Automation}, pages 763--768, 2009.
\newblock \doi{10.1109/ROBOT.2009.5152385}.

\bibitem[Peng et~al.(2021)Peng, Kumar, Zhang, and Levine]{peng2021advantageweighted}
X.~B. Peng, A.~Kumar, G.~Zhang, and S.~Levine.
\newblock Advantage-weighted regression: Simple and scalable off-policy reinforcement learning, 2021.
\newblock URL \url{https://openreview.net/forum?id=ToWi1RjuEr8}.

\bibitem[Ross et~al.(2011)Ross, Gordon, and Bagnell]{daggar}
S.~Ross, G.~Gordon, and D.~Bagnell.
\newblock A reduction of imitation learning and structured prediction to no-regret online learning.
\newblock In G.~Gordon, D.~Dunson, and M.~Dudík, editors, \emph{Proceedings of the Fourteenth International Conference on Artificial Intelligence and Statistics}, volume~15 of \emph{Proceedings of Machine Learning Research}, pages 627--635, Fort Lauderdale, FL, USA, 11--13 Apr 2011. PMLR.
\newblock URL \url{https://proceedings.mlr.press/v15/ross11a.html}.

\bibitem[Schaal(1996)]{Mozer1996Learning}
S.~Schaal.
\newblock Learning from demonstration.
\newblock In M.~Mozer, M.~Jordan, and T.~Petsche, editors, \emph{Advances in Neural Information Processing Systems}, volume~9. MIT Press, 1996.
\newblock URL \url{https://proceedings.neurips.cc/paper_files/paper/1996/file/68d13cf26c4b4f4f932e3eff990093ba-Paper.pdf}.

\bibitem[Sohn et~al.(2015)Sohn, Lee, and Yan]{cvae}
K.~Sohn, H.~Lee, and X.~Yan.
\newblock Learning structured output representation using deep conditional generative models.
\newblock In C.~Cortes, N.~Lawrence, D.~Lee, M.~Sugiyama, and R.~Garnett, editors, \emph{Advances in Neural Information Processing Systems}, volume~28. Curran Associates, Inc., 2015.
\newblock URL \url{https://proceedings.neurips.cc/paper_files/paper/2015/file/8d55a249e6baa5c06772297520da2051-Paper.pdf}.

\bibitem[Sutton and Barto(2018)]{Sutton1998RL}
R.~S. Sutton and A.~G. Barto.
\newblock \emph{Reinforcement Learning: An Introduction}.
\newblock The MIT Press, second edition, 2018.
\newblock URL \url{http://incompleteideas.net/book/the-book-2nd.html}.

\bibitem[Tarasov et~al.(2023)Tarasov, Kurenkov, Nikulin, and Kolesnikov]{rebrac}
D.~Tarasov, V.~Kurenkov, A.~Nikulin, and S.~Kolesnikov.
\newblock Revisiting the minimalist approach to offline reinforcement learning.
\newblock In \emph{Thirty-seventh Conference on Neural Information Processing Systems}, 2023.
\newblock URL \url{https://openreview.net/forum?id=vqGWslLeEw}.

\bibitem[Wang et~al.(2018)Wang, Xiong, Han, sun, Liu, and Zhang]{wang2018Advances}
Q.~Wang, J.~Xiong, L.~Han, p.~sun, H.~Liu, and T.~Zhang.
\newblock Exponentially weighted imitation learning for batched historical data.
\newblock In S.~Bengio, H.~Wallach, H.~Larochelle, K.~Grauman, N.~Cesa-Bianchi, and R.~Garnett, editors, \emph{Advances in Neural Information Processing Systems}, volume~31. Curran Associates, Inc., 2018.
\newblock URL \url{https://proceedings.neurips.cc/paper_files/paper/2018/file/4aec1b3435c52abbdf8334ea0e7141e0-Paper.pdf}.

\bibitem[Wu et~al.(2020)Wu, Tucker, and Nachum]{wu2020behavior}
Y.~Wu, G.~Tucker, and O.~Nachum.
\newblock Behavior regularized offline reinforcement learning, 2020.
\newblock URL \url{https://openreview.net/forum?id=BJg9hTNKPH}.

\bibitem[Wu et~al.(2021)Wu, Zhai, Srivastava, Susskind, Zhang, Salakhutdinov, and Goh]{uwac}
Y.~Wu, S.~Zhai, N.~Srivastava, J.~M. Susskind, J.~Zhang, R.~Salakhutdinov, and H.~Goh.
\newblock Uncertainty weighted offline reinforcement learning, 2021.
\newblock URL \url{https://openreview.net/forum?id=7hMenh--8g}.

\bibitem[Xu et~al.(2022)Xu, Jiang, Li, and Zhan]{xu2022apolicyguided}
H.~Xu, L.~Jiang, J.~Li, and X.~Zhan.
\newblock A policy-guided imitation approach for offline reinforcement learning.
\newblock In A.~H. Oh, A.~Agarwal, D.~Belgrave, and K.~Cho, editors, \emph{Advances in Neural Information Processing Systems}, 2022.
\newblock URL \url{https://openreview.net/forum?id=CKbqDtZnSc}.

\bibitem[Xu et~al.(2021)Xu, Zhang, Li, Du, Kawarabayashi, and Jegelka]{xu2021how}
K.~Xu, M.~Zhang, J.~Li, S.~S. Du, K.-I. Kawarabayashi, and S.~Jegelka.
\newblock How neural networks extrapolate: From feedforward to graph neural networks.
\newblock In \emph{International Conference on Learning Representations}, 2021.
\newblock URL \url{https://openreview.net/forum?id=UH-cmocLJC}.

\bibitem[Yu et~al.(2021)Yu, Kumar, Rafailov, Rajeswaran, Levine, and Finn]{combo}
T.~Yu, A.~Kumar, R.~Rafailov, A.~Rajeswaran, S.~Levine, and C.~Finn.
\newblock Combo: Conservative offline model-based policy optimization.
\newblock In M.~Ranzato, A.~Beygelzimer, Y.~Dauphin, P.~Liang, and J.~W. Vaughan, editors, \emph{Advances in Neural Information Processing Systems}, volume~34, pages 28954--28967. Curran Associates, Inc., 2021.
\newblock URL \url{https://proceedings.neurips.cc/paper_files/paper/2021/file/f29a179746902e331572c483c45e5086-Paper.pdf}.

\bibitem[Yue et~al.(2023)Yue, Lu, Kang, Song, and Huang]{LNExplanation}
Y.~Yue, R.~Lu, B.~Kang, S.~Song, and G.~Huang.
\newblock Understanding, predicting and better resolving q-value divergence in offline-{RL}.
\newblock In \emph{Thirty-seventh Conference on Neural Information Processing Systems}, 2023.
\newblock URL \url{https://openreview.net/forum?id=71P7ugOGCV}.

\bibitem[Zhang et~al.(2021)Zhang, Kuppannagari, and Viktor]{Zhang2021BRAC}
C.~Zhang, S.~Kuppannagari, and P.~Viktor.
\newblock Brac+: Improved behavior regularized actor critic for offline reinforcement learning.
\newblock In V.~N. Balasubramanian and I.~Tsang, editors, \emph{Proceedings of The 13th Asian Conference on Machine Learning}, volume 157 of \emph{Proceedings of Machine Learning Research}, pages 204--219. PMLR, 17--19 Nov 2021.
\newblock URL \url{https://proceedings.mlr.press/v157/zhang21a.html}.

\end{thebibliography}

\newpage
\appendix

\section{Proofs}
\label{sec:appendix-proofs}

\subsection{Proof of Theorem \ref{thm:scq-pointwise}} \label{sec:proof-point-wise}
We begin by proving Theorem \ref{thm:scq-pointwise}, demonstrating that the Q values obtained by SCQ serve as $\varepsilon$-point-wise lower bounds of the true Q values. The methodology of this proof follows similar steps as outlined in \cite{cql}. It is important to note that this proof is conducted without accounting for sampling errors introduced by the Bellman update. However, the proof can be readily extended to accommodate an approximate Bellman update by incorporating sampling error terms.

\begin{proof}
    By the assumption, the estimated Q function is represented as a linear function, and it can be written as 
    \begin{equation}
        \hat{Q}^k(s, a) = \langle \phi(s, a), \hat{w}^k \rangle
    \end{equation}
    The optimal weight $\hat{w}^k$ can be obtained by solving the following optimization problem.
    \begin{equation}
    \underset{Q}{\operatorname{min}} \; \alpha_k  \, \mathbb{E}_{s \sim D,a \sim A_{ood}} \left[ Q(s,a) \right] + \frac{1}{2} \mathbb{E}_{s,a,s' \sim D} \left[ \left( Q(s,a) - B^{\pi} {\hat{Q}^k}(s,a) \right)^2 \right]
    \end{equation}
    By substituting $Q(s,a)=w^{\top} \phi(s,a)$, and setting the derivative with respect to $w$ to be $0$, we get
    \begin{equation}
        \alpha_k \sum_{s,a} d^{\pi_{\beta}}(s) \pi_{\text{ood}}(a|s) \phi(s, a) + \sum_{s,a} d^{\pi_{\beta}}(s) \pi_{\beta}(a|s) \left(Q(s,a) - B^{\pi}\hat{Q}^k(s,a)\right) \phi(s,a) = 0
    \end{equation}
    where $\pi_{\text{ood}}(a|s) = \pi(a|s) \mathds{1}_{a\in A_{\text{ood}}}$.
    Let's define the following two new matrices:
    \begin{equation}
        \begin{split}
            D & =\mathrm{diag}(d^{\pi_{\beta}}(s) \pi_{\beta}(a|s) \in \mathbb{R}^{|S||A| \times |S||A|} \\
            \Phi &= \left[\phi(s_1, a_1), \phi(s_2, a_2), \cdots, \phi(s_n, a_n)\right]^{\top} \in \mathbb{R}^{|S||A| \times d}
        \end{split}
    \end{equation}
    Here $|S|$ and $|A|$ are number of states and actions that are in the dataset. By rearranging terms with vectorization, $w_{k+1}$ can be written as:
    \begin{equation}
    \begin{split}
        (\Phi^{\top} D\Phi) \hat{w}^{k+1} &= \Phi^{\top} D \left( B^{\pi}\hat{Q}^k(s,a) \right) - \alpha_k \Phi^{\top}\mathrm{diag}\left[ d^{\pi_{\beta}}(s)  \pi_{ood}(a|s) \right] \\
        &= \Phi^{\top} D \left(B^{\pi}\hat{Q}^k(s,a)\right) - \alpha_k \Phi^{\top} D  \frac{\pi_{ood}(a|s)}{\pi_{\beta}(a|s)} \mathbf{1}
    \end{split}
    \end{equation}
    where $\mathbf{1} \in \mathbb{R}^{_{|S||A|\times 1}}$ is a vector that all entries are $1$.
    Thus the estimated Q values at time step $k+1$ is 
    \begin{equation} \label{eq:esitamted-q-relationship}
    \begin{split}
        \hat{Q}^{k+1} &= \Phi \hat{w}^{k+1} \\
        &= \Phi \left(\Phi^{\top}D\Phi\right)^{-1} \Phi^{\top} D \left( B^{\pi}\hat{Q}^k \right) - \alpha_k \Phi\left(\Phi^{\top}D\Phi\right)^{-1}\Phi^{\top} D \frac{\pi_{ood}(a|s)}{\pi_{\beta}(a|s)} \mathbf{1} \\
        &= \hat{Q}^{k+1}_{\text{LSTD-Q}}  - \alpha_k \Phi(\Phi^{\top}D\Phi)^{-1}\Phi^{\top} D \frac{\pi_{ood}(a|s)}{\pi_{\beta}(a|s)} \mathbf{1} 
    \end{split}
    \end{equation}
    where $\hat{Q}^{k+1}_{\text{LSTD-Q}}=\Phi (\Phi^{\top}D\Phi)^{-1} \Phi^{\top} D (B^{\pi}\hat{Q}^k)$ is estimated Q values obtained from the least-squares temporal difference Q-learning as LSTD-Q\cite{lstdq}. Next we consider the two cases where $\pi_{\text{ood}} \neq 0$ and  $\pi_{\text{ood}} = 0$.
    
    \paragraph{Case 1. $\pi_{\text{ood}}(a|s) \neq 0$}
    In this case, action is sampled from the OOD action sets $A_{\text{ood}}$. Since $\alpha_k \geq 0$ and $ \Phi(\Phi^{\top}D\Phi)^{-1}\Phi^{\top} D \frac{\pi_{ood}(a|s)}{\pi_{\beta}(a|s)} \geq 0$, the following inequality is satisfied for all states $s$ and actions $a$.
    \begin{equation} \label{eq:scq-lower-lstdq}
        \hat{Q}^{k+1}(s, a) \leq \hat{Q}^{k+1}_{\text{LSTD-Q}}(s,a)
    \end{equation}
    Our goal is to show that the estimated Q values $\hat{Q}^{k}(s, a)$ become point-wise lower bound of the true Q values $Q(s, a)$. So far we only show \eqref{eq:scq-lower-lstdq}, which states the estimated Q values from SCQ are point-wise lower bounds of the Q values obtained by the LSTD-Q algorithm. However, due to the interpolation error, $\hat{Q}^{k+1}_{\text{LSTD-Q}}(s,a)$ can be larger than the true Q values. Therefore, we need to choose $\alpha_k$ such that the pessimism term can offset the overestimation caused by $\hat{Q}^{k+1}_{\text{LSTD-Q}}(s,a)$. 
    \begin{equation}
        \begin{split}
            \hat{Q}^{k+1}(s,a) &\leq \hat{Q}^{k+1}_{\text{LSTD-Q}}(s,a)  - \alpha_k \Phi(\Phi^{\top}D\Phi)^{-1}\Phi^{\top} D \frac{\pi_{ood}(a|s)}{\pi_{\beta}(a|s)} \mathbf{1} \\
            &\leq Q(s,a) - \left( Q(s,a) - \hat{Q}^{k+1}_{\text{LSTD-Q}} \right)  - \alpha_k \Phi(\Phi^{\top}D\Phi)^{-1}\Phi^{\top} D \frac{\pi_{ood}(a|s)}{\pi_{\beta}(a|s)} \mathbf{1} \\
            & \leq Q(s,a)
        \end{split}
    \end{equation}
    where $\alpha_k$ is 
    \begin{equation} \label{eq:alpha-condition-pointwise}
        \alpha_k \geq \text{max} \left(\frac{\hat{Q}^{k+1}_{\text{LSTD-Q}} - Q(s,a)}{\Phi(\Phi^{\top}D\Phi)^{-1}\Phi^{\top} D \frac{\pi_{ood}(a|s)}{\pi_{\beta}(a|s)}\mathbf{1}} ,0\right)
    \end{equation}
    Therefore, when $\pi_{\text{ood}} \neq 0$, the estimated Q values become point-wise lower bound with $\alpha$ satisfies \eqref{eq:alpha-condition-pointwise}.
    
    \paragraph{Case 2. $\pi_{\text{ood}}(a|s) = 0$}
    In this case, the sampled action $a$ comes from the dataset or around the dataset. From the \eqref{eq:esitamted-q-relationship}, we can get $\hat{Q}_{k+1}(s,a) = \hat{Q}^{k+1}_{\text{LSTD-Q}}(s,a)$. With Assumption\ref{ass:epsilon-bound}, we can get
    \begin{equation}
    \begin{split}
            \hat{Q}^{k+1}(s,a) &= \hat{Q}^{k+1}_{\text{LSTD-Q}}(s,a) \\
            &= Q(s,a) + \left( \hat{Q}^{k+1}_{\text{LSTD-Q}}(s,a) - Q(s,a) \right) \\
            &= Q(s,a) + \left( \hat{Q}^{k+1}(s,a) - Q(s,a) \right) \\
            &\leq Q(s,a) + \varepsilon \quad \left(a \notin A_{\text{ood}}\right)\\
    \end{split}
    \end{equation}
    Thus, the estimated Q values become $\varepsilon$-point-wise lower-bound of the Q value.
\end{proof}

\subsection{Proof of Theorem \ref{thm:scq-state-value}} \label{sec:proof-state-value}
Here we use the same characters as the same meaning in Appendix\ref{sec:proof-point-wise}.
In this proof, we assume that $\pi_{\text{ood}}(a|s) \neq 0$ for all states $s$ and actions $a$. We first prove the following proposition.

\begin{prop} \label{prop:fpi-difference}
Define the function $f(\pi)$ and $f_{\text{ood}}(\pi)$ in the following way.
    \begin{equation}
    \begin{split}
        f(\pi) &= \pi(a|s)^T P_{\Phi} \left[ \frac{\pi(a|s) - \pi_{\beta}(a|s)}{\pi_{\beta}(a|s)} \right] \\
        f_{\text{ood}}(\pi) &= \pi(a|s)^T P_{\Phi} \left[ \frac{\pi_{\text{ood}}(a|s)}{\pi_{\beta}(a |s)} \right] \\
    \end{split}
    \end{equation}
    where $\pi_{\text{ood}}(a|s) = \pi(a|s) \mathds{1}_{a\in A_{\text{ood}}}$ and $P_{\Phi} := \Phi \left( \Phi^T D \Phi \right)^{-1} \Phi^T D$. Therefore, there exists $\tau$ that satisfies the following inequality under the condition $0 < \tau \leq 1$.
    \begin{equation}
        f(\pi) > \tau f_{\text{ood}}(\pi)
    \end{equation}
\end{prop}

\begin{proof}
    We first decompose the learned policy $\pi(a|s)$ in the following way.
    \begin{equation}
        \pi(a|s) = \pi_{\text{idd}}(a|s) + \pi_{\text{ood}}(a|s) 
    \end{equation}
    where $\pi_{\text{idd}}(a|s) = \pi(a|s) \mathds{1}_{a\notin A_{\text{ood}}}$ and $\pi_{\text{ood}}(a|s) = \pi(a|s) \mathds{1}_{a\in A_{\text{ood}}}$. From the proof of CQL\cite{cql}, we already know $f(\pi) \geq 0$ and it achieves the minimum value $f(\pi) = 0$ when $\pi(a|s) = \pi^{*}(a|s)$. Moreover, we can tell $P_{\Phi} > 0$ from its definition, which leads to $f_{\text{ood}}(\pi) > 0$. Subtracting $\tau f_{\text{ood}}(\pi)$ from $f(\pi)$, we get:
    \begin{equation}
        \begin{split}
            f(\pi) - \tau f_{\text{ood}}(\pi) &=  \pi(a|s)^T P_{\Phi} \left[ \frac{\pi(a|s) - \pi_{\beta}(a|s) - \tau \pi_{\text{ood}}(a|s)}{\pi_{\beta}(a|s)} \right] \\
            &=  \pi(a|s)^T P_{\Phi} \left[ \frac{\pi_{\text{idd}}(a|s) - \pi_{\beta}(a|s) + (1 - \tau) \pi_{\text{ood}}(a|s)}{\pi_{\beta}(a|s)} \right] \\
            &=  \pi(a|s)^T P_{\Phi} \left[ \frac{\pi_{\text{idd}}(a|s) - \pi_{\beta}(a|s)}{\pi_{\beta}(a|s)}\right] +  \pi(a|s)^T P_{\Phi}  \left[ \frac{(1-\tau)\pi_{\text{ood}}(a|s)}{\pi_{\beta}(a|s)}\right] \\
            &= f_{\text{idd}}(\pi) + (1-\tau) f_{\text{ood}}(\pi)
        \end{split}
    \end{equation}
    where $f_{\text{idd}}(\pi) = \pi(a|s)^T P_{\Phi} \left[ \frac{\pi_{\text{idd}}(a|s) - \pi_{\beta}(a|s)}{\pi_{\beta}(a|s)}\right]$. Since $f_{\text{ood}}(\pi)$ is always positive, we consider the following two cases, where $f_{\text{idd}}(\pi)$ becomes either positive or negative.

    \paragraph{Case 1. $f_{\text{idd}}(\pi)\geq0$}
    In this case, both $f_{\text{idd}}(\pi)$ and $f_{\text{ood}}(\pi)$ are positive, thus $f(\pi) - \tau f_{\text{ood}}(\pi)\geq0$ under the condition $0 < \tau \leq 1$. 

    \paragraph{Case 2. $f_{\text{idd}}(\pi) < 0$}
    In this case, it needs to satisfy $-f_{\text{idd}}(\pi) \leq (1-\tau) f_{\text{ood}}(\pi)$, which becomes 
    \begin{equation}
        \begin{split}
            &\qquad -f_{\text{idd}}(\pi) \leq (1-\tau) f_{\text{ood}}(\pi) \\
            &\Longleftrightarrow -(f_{\text{idd}}(\pi) + f_{\text{ood}}(\pi)) \leq -\tau f_{\text{ood}}(\pi) \\
            &\Longleftrightarrow f(\pi) \geq \tau f_{\text{ood}}(\pi) \\
            &\Longleftrightarrow \tau \leq \frac{f(\pi)}{f_{\text{ood}}(\pi) }
        \end{split}
    \end{equation}
    Since $\frac{f(\pi)}{f_{\text{ood}}(\pi) } \geq 0$, we can find a parameter $\tau$ in $0<\tau\leq 1$ to make $f(\pi) - \tau f_{\text{ood}}(\pi)\geq0$.

    Therefore, the inequality $f(\pi) - \tau f_{\text{ood}}(\pi)\geq0$ is always true when $0 < \tau \leq 1$.
\end{proof}

With Preposition\ref{prop:fpi-difference}, we can proof Theorem\ref{thm:scq-state-value}.
\begin{proof}
    We define the tuning parameter for CQL in \eqref{eq:cql-q-update} as $\alpha_{\text{k, cql}}$. We also define the parameter for SCQ in \eqref{eq:scq-q-update} as $\alpha_{k, \text{scq}} = \tau \alpha_{k, \text{cql}}$ where $0 < \tau \leq 1$. The estimated state value $V_{\text{SCQ}}^{k+1}$ from SCQ can be calculated from \eqref{eq:esitamted-q-relationship} as:
    \begin{equation} \label{eq:state-value-scq}
        \begin{split}
            V_{\text{SCQ}}^{k+1}(s) &= \pi(a|s)^{\top} \hat{Q}^{k+1}_{\text{SCQ}}(s,a) \\
            &=  \pi(a|s)^{\top} \hat{Q}^{k+1}_{\text{LSTD-Q}}(s,a) - \alpha_{k, \text{scq}} f_{\text{ood}}(\pi) \\
            &= \hat{V}^{k+1}_{\text{LSTD-Q}} - \alpha_{k, \text{scq}} f_{\text{ood}}(\pi) 
        \end{split}
    \end{equation}
    In addition, we can show $V_{\text{SCQ}}^{k+1}$ is lower bound of the true state value $V^{k+1}(s)$ at every state.
    \begin{equation} \label{eq:scq-state-upper-bound}
        \begin{split}
             V_{\text{SCQ}}^{k+1}(s) &= \hat{V}^{k+1}_{\text{LSTD-Q}}(s) - \alpha_{k, \text{scq}} f_{\text{ood}}(\pi) \\
             &= V^{k+1}(s) - (V^{k+1}(s) - \hat{V}^{k+1}_{\text{LSTD-Q}}(s))  - \alpha_{k, \text{scq}} f_{\text{ood}}(\pi) \\
             &\leq V^{k+1}(s)
        \end{split}
    \end{equation}
    where $\alpha_{k, \text{scq}} = \text{max} \; (\frac{\hat{V}^{k+1}_{\text{LSTD-Q}}(s) - V^{k+1}(s)}{f_{\text{ood}}(\pi)}, 0)$.
    
    Using the Theorem.D.1 in CQL\cite{cql}, \eqref{eq:state-value-scq} and \eqref{eq:scq-state-upper-bound}, we get the following inequality. 
    \begin{equation}
    \begin{split}
        \hat{V}^{k+1}_{\text{CQL}}(s) &\leq \hat{V}^{k+1}_{\text{LSTD-Q}}(s) - \alpha_{k, \text{cql}} f(\pi) \\
        &\leq  \hat{V}^{k+1}_{\text{LSTD-Q}}(s) - \alpha_{k, \text{cql}} \tau f_{\text{ood}}(\pi) \\
        &= \hat{V}^{k+1}_{\text{LSTD-Q}}(s) -  \alpha_{k, \text{scq}} f_{\text{ood}}(\pi) \\
        &= V_{\text{SCQ}}^{k+1}(s)
    \end{split}
    \end{equation}
    Therefore, $\hat{V}^{k+1}_{\text{CQL}}(s) \leq V_{\text{SCQ}}^{k+1}(s) \leq V^{k+1}(s)$ for all states $s$. Note that from the first row to the second row, we use the Preposition\ref{prop:fpi-difference}.
\end{proof}
This proof suggests that if $\alpha_{k, \text{cql}}$ satisfies the Theorem.D.1 in CQL paper, we can always find $\tau$ that makes the state value $V_{\text{SCQ}}^{k+1}(s)$ larger than $V_{\text{CQL}}^{k+1}(s)$.

\section{Experimental details} \label{sec:appendix-experimental-details}
In our experiment, we use Soft-Actor-Critic (SAC) \cite{sac, sac2, sac3} as our baseline for the implementation. However, SCQ can be incorporated into other model-free RL algorithms, e.g. TD3\cite{td3}. Our network follows a similar structure described in \cite{mcq}. We use Pytorch and Python=3.8 to implement the algorithm and train it with Nvidia GeForce RTX 4090. 

We summarize the hyperparameters in Table.\ref{tab:experiment-parameter-details}. In the experiment, CVAE is employed to distinguish between out-of-distribution (OOD) and in-distribution (IDD) actions, maintaining a minimalistic architecture for the CVAE to ensure clarity of the underlying processes. In addition, We find using a cosine annealing scheduler \cite{sgdr} for the actor's learning rate helps SCQ to increase its stability and performance. Note that further fine-tuning of the learning rates and modifications to the architecture of the network may lead to enhanced model robustness and performance optimization.

Table.\ref{tab:experiment-alpha} also describes the value of $\alpha$ used in Eq.\ref{eq:scq-q-update}. This value decides the conservatives of the learned Q-values, and it highly affects the performance of the algorithm. One disadvantage of SCQ is that it needs to change $\alpha$ depending on the environment. However, we believe that this hyperparameter $\alpha$ can be automatically tuned via Lagrangian dual gradient descent used in \cite{cql}. More implementation details can be found in our official code.

\begin{table}[htbp]
\centering
\caption{Hyperparameters setup for SCQ.}
\label{tab:experiment-parameter-details}
\begin{tabular}{@{}lc@{}}
\toprule
\textbf{Hyperparameter} & \textbf{Value} \\
\midrule
\multicolumn{2}{l}{SAC} \\
Dimension of actor hidden layer & 400 \\
Dimension of critic hidden layer & 400 \\
Number of actor hidden layers & 2 \\
Number of critic hidden layers & 2 \\
Nonlinearity & ReLU \\
Batch size & 256 \\
Critic learning rate & $3 \times 10^{-4}$ \hspace{1mm} ($1 \times 10^{-4}$ for Antmaze)\\
Actor learning rate & $3 \times 10^{-4}$ \\
Optimizer & Adam \cite{adam} \\
Discount factor & 0.99 \\
Maximum log std & 2 \\
Minimum log std & $-3$ \\
Use automatic entropy tuning & Yes \\
Use actor learning rate scheduler & Yes (scheduler \cite{sgdr}) \\
Target update rate & $5 \times 10^{-3}$ \\
\midrule
\multicolumn{2}{l}{CVAE} \\
Encoder \& Decoder hidden dimension & 750 \\
Number of Hidden layers & 1 \\
Nonlinearity & ReLU \\
CVAE learning rate & $1 \times 10^{-3}$ \\
Batch size & 256 \\
Latent dimension & $2\times$ action dimension \\
\bottomrule
\end{tabular}
\end{table}

\begin{table}[htbp]
\centering
\caption{SCQ hyperparameter $\alpha$ used experiments on D4RL MuJoCo-Gym and Antmaze "-v2" datasets.}
\label{tab:experiment-alpha}
\begin{tabular}{lcc}
\toprule
\textbf{Task Name} & \textbf{critic hyperparameter} $\alpha$  & \textbf{actor hyperparameter} $\beta$\\
\midrule
halfcheetah-random & 0.1  & -\\
hopper-random      & 1.0  & - \\
walker2d-random    & 15.0 & -\\
halfcheetah-medium & 0.05 & -\\
hopper-medium      & 2.5  & -\\
walker2d-medium    & 2.0  & -\\
halfcheetah-medium-replay  & 0.2  & -\\
hopper-medium-replay       & 1.0  & -\\
walker2d-medium-replay     & 2.0  & - \\
halfcheetah-medium-expert  & 4.0  & - \\
hopper-medium-expert       & 15.0 & - \\
walker2d-medium-expert     & 1.5  & -\\
halfcheetah-expert         & 5.0  & -\\
hopper-expert              & 10.0 & - \\
walker2d-expert            & 1.0  & -\\
antmaze-umaze         & 2.0   & -\\
antmaze-umaze-diverse & 0.3   & 1.0 \\
antmaze-medium-play    & 0.1  & 1.0 \\
antmaze-medium-diverse & 1.0  & 1.0 \\
antmaze-large-play    & 0.1   & 0.5\\
antmaze-large-diverse & 0.05  & 0.5\\
\bottomrule
\end{tabular}
\end{table}

\section{Additional Ablation studies (Parameter changes)} \label{sec:appendix-additional-ablation-studies}
In addition to the ablation studies in Section.\ref{sec:exp-ablation-study}, we change the parameter $\alpha$ in En.\eqref{eq:scq-q-update} to see the performance difference. Table.\ref{tab:exp-extra-ablation} shows the parameter change experiment results. The highest score with optimal parameter $\alpha$ is highlighted in bold. From this result, we can tell if $\alpha$ is too low (e.g $\alpha=0$), the performance of the algorithm deteriorates. Moreover, when we make $\alpha$ larger, the score also gets lower. This is because CVAE separation of OOD and IDD actions is not completely accurate, and thus SCQ lowers Q-values even at IDD points. We also show the learned average Q values in Fig.\ref{fig:ablation-halfcheetah-q}. From this figure, we can tell that Q values go up when we set small $\alpha$. On the other hand, when we set large values for $\alpha$, the learned Q values get small due to the over-penalization of Q values. Therefore, choosing the optimal hyper-parameter $\alpha$ plays an important role in SCQ to get high performance. 

\begin{table}[h]
\caption{Parameter change results of SCQ on halfCheetah tasks. We evaluated each method using the "-v2" datasets across 5 random seeds for the non-optimal parameters. The label abbreviation follows the same order as Table.\ref{tab:mujoco-result}.}
\label{tab:exp-extra-ablation}
\centering
\begin{tabular}{lccccc}
\toprule
\diagbox{Tasks}{$\alpha$} & 0.0 & 1.0 & 10.0 & optimal \\
\midrule
halfcheetah-m   & 48.3 $\pm$ 19.2 & 60.3 $\pm$ 0.4 & 51.1 $\pm$ 0.4 & 46.5 $\pm$ 0.3  & \textbf{68.3 $\pm$ 1.64} \\
halfcheetah-m-r & 37.5 $\pm$ 12.3 & 50.7 $\pm$ 0.4 & 46.0 $\pm$ 0.7 & 43.7 $\pm$ 0.2  & \textbf{57.2 $\pm$ 1.78} \\
halfcheetah-m-e & 16.6 $\pm$ 6.7  & 92.2 $\pm$ 8.5 & \textbf{98.8 $\pm$ 0.6} & 96.8 $\pm$ 0.3 & \textbf{98.8 $\pm$ 0.6} \\
\bottomrule
\end{tabular}
\end{table}

\begin{figure}[h]
  \centering
  \includegraphics[scale=0.35]{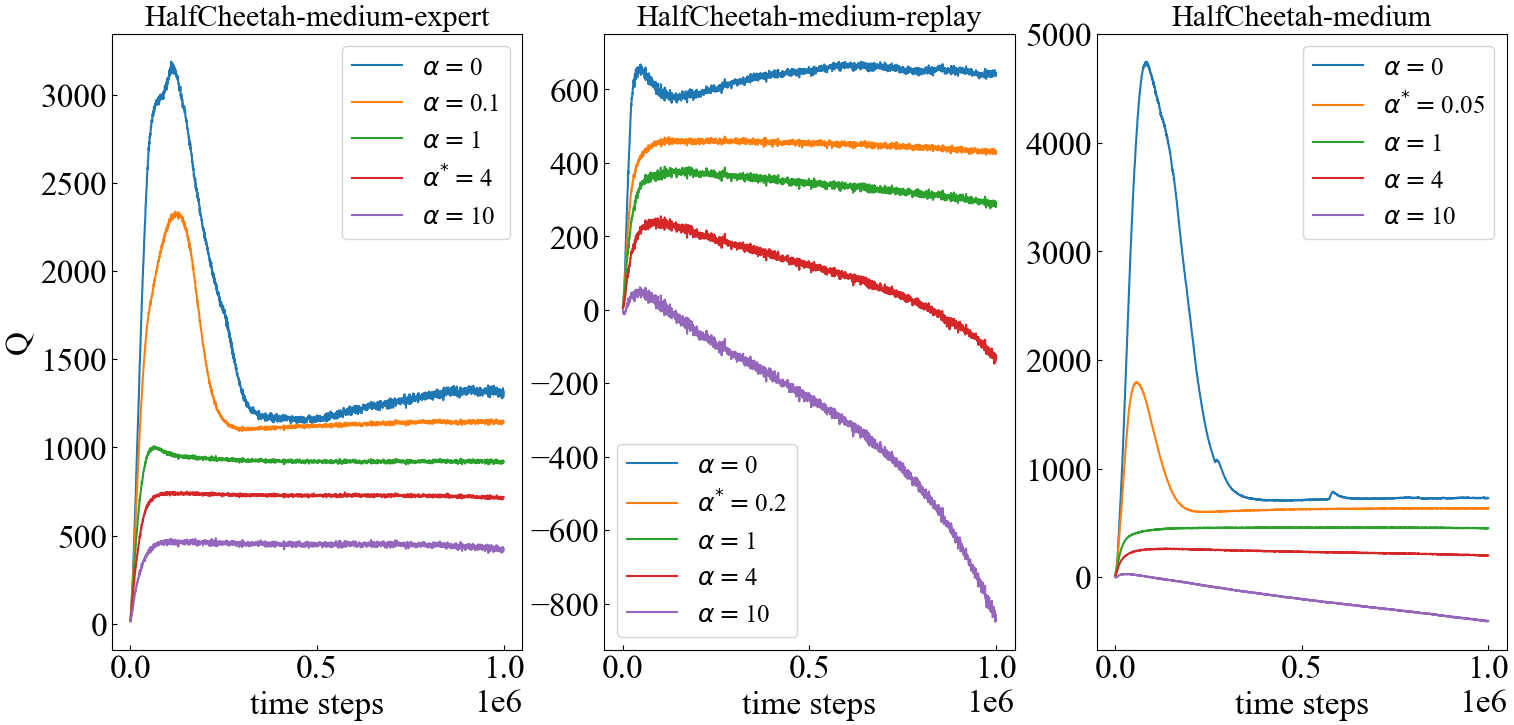}
  \caption{Comparison of Q values with different $\alpha$.}
  \label{fig:ablation-halfcheetah-q}
\end{figure}

\section{Learning Curve}
In this section, we show the learning curve (normalized score over the whole learning time). We illustrate the learning curve in Fig.\ref{fig:learning-score}.

\begin{figure}[h]
  \centering
  \includegraphics[scale=0.365]{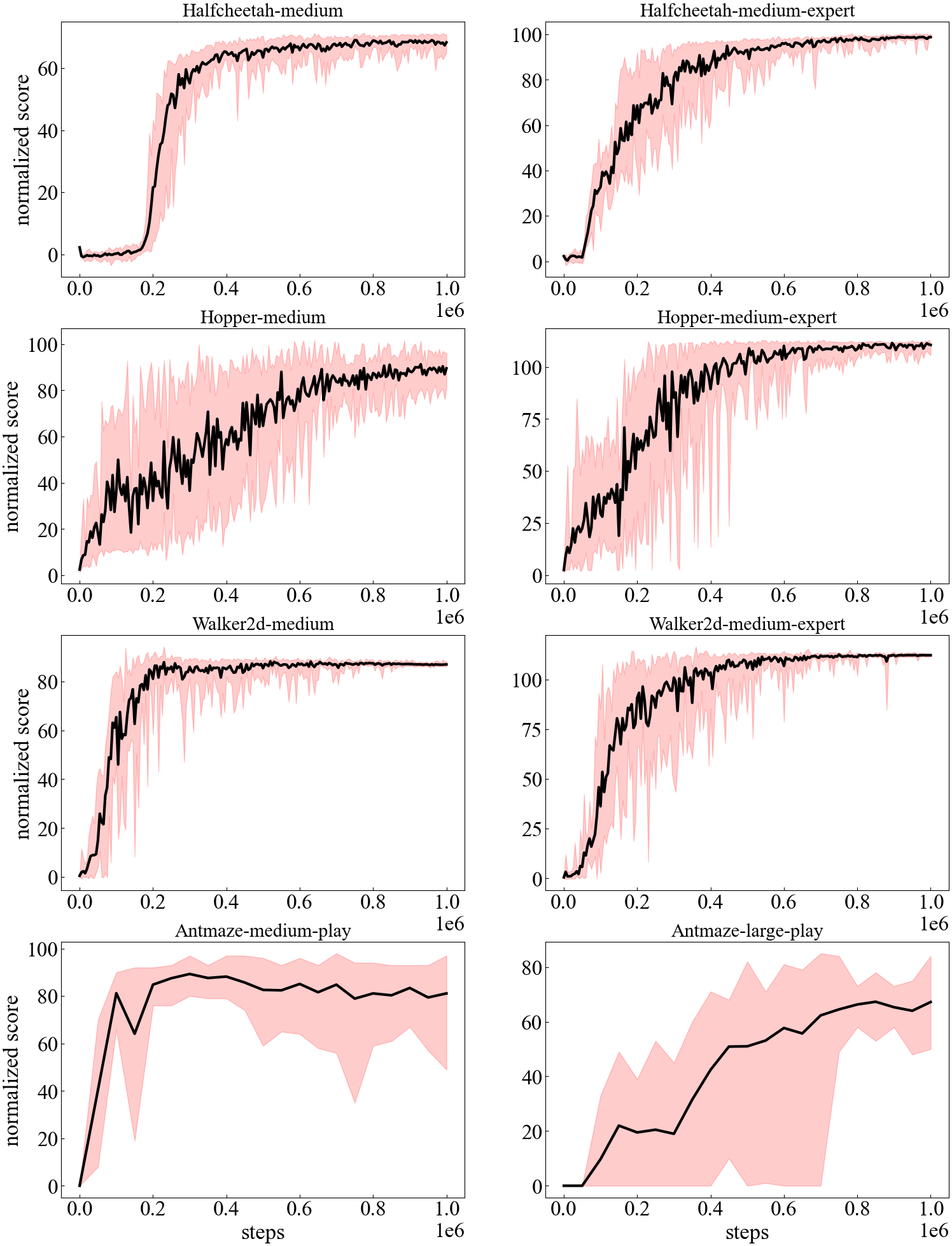}
  \caption{Learning curve over during the training. The black line shows the average normalized score over 10 seeds and the red range area shows the \textbf{maximum and minimum} values at each time step.}
  \label{fig:learning-score}
\end{figure}

\end{document}